\documentclass[twoside]{article}
\usepackage[accepted]{aistats2023}
\usepackage{amsfonts}
\usepackage{amsmath}
\usepackage{amsthm}
\usepackage{booktabs}
\usepackage{float}
\usepackage{hyperref}
\usepackage{microtype} 
\usepackage{natbib}
\usepackage{nicefrac}
\usepackage{subfig}
\usepackage{tikz}
\usepackage{url}
\usepackage{xcolor}

\newcommand{\R}{\mathbb{R}}
\newcommand{\E}{\mathbf{E}}
\newcommand{\entropy}{\mathbf{H}}
\newcommand{\I}{\mathbb{I}}
\newcommand{\X}{\mathbb{X}}
\newcommand{\D}{\mathcal{D}}
\newcommand{\prob}{\mathbf{P}}
\newcommand{\eubo}{\mathrm{qEUBO}}
\newcommand{\qei}{\mathrm{qEI}}

\DeclareMathOperator*{\argmax}{argmax}

\newtheorem{theorem}{Theorem}

\makeatletter
\def\renewtheorem#1{%
  \expandafter\let\csname#1\endcsname\relax
  \expandafter\let\csname c@#1\endcsname\relax
  \gdef\renewtheorem@envname{#1}
  \renewtheorem@secpar
}
\def\renewtheorem@secpar{\@ifnextchar[{\renewtheorem@numberedlike}{\renewtheorem@nonumberedlike}}
\def\renewtheorem@numberedlike[#1]#2{\newtheorem{\renewtheorem@envname}[#1]{#2}}
\def\renewtheorem@nonumberedlike#1{  
\def\renewtheorem@caption{#1}
\edef\renewtheorem@nowithin{\noexpand\newtheorem{\renewtheorem@envname}{\renewtheorem@caption}}
\renewtheorem@thirdpar
}
\def\renewtheorem@thirdpar{\@ifnextchar[{\renewtheorem@within}{\renewtheorem@nowithin}}
\def\renewtheorem@within[#1]{\renewtheorem@nowithin[#1]}
\makeatother
\begin{document}

%
\runningtitle{qEUBO: A Decision-Theoretic Acquisition Function for Preferential Bayesian Optimization}

%

\twocolumn[

\aistatstitle{qEUBO: A Decision-Theoretic Acquisition Function for\\ Preferential Bayesian Optimization}

\aistatsauthor{Raul Astudillo \And Zhiyuan Jerry Lin \And Eytan Bakshy \And Peter I. Frazier}

\aistatsaddress{Caltech \And  Meta \And Meta \And Cornell University} ]

\begin{abstract}
 Preferential Bayesian optimization (PBO) is a framework for optimizing a decision maker's  latent utility function using preference feedback. This work introduces the expected utility of the best option (qEUBO) as a novel acquisition function for PBO. When the decision maker's responses are noise-free, we show that qEUBO is one-step Bayes optimal and thus equivalent to the popular knowledge gradient acquisition function. We also show that qEUBO enjoys an additive constant approximation guarantee to the one-step Bayes-optimal policy when the decision maker's responses are corrupted by noise. We provide an extensive evaluation of qEUBO and demonstrate that it outperforms the state-of-the-art acquisition functions for PBO across many settings. Finally, we show that, under sufficient regularity conditions, qEUBO's Bayesian simple regret converges to zero at a rate $o(1/n)$ as the number of queries, $n$, goes to infinity. 
 In contrast, we show that simple regret under qEI, a popular acquisition function for standard BO often used for PBO, can fail to converge to zero. 
 Enjoying superior performance, simple computation, and a grounded decision-theoretic justification, qEUBO is a promising acquisition function for PBO.
\end{abstract}

\section{INTRODUCTION}
Bayesian optimization (BO) is a framework for global optimization of objective functions with expensive or time-consuming evaluations \citep{shahriari2015taking}. BO algorithms have been successful in broad range of applications, such as sensor set selection \citep{garnett2010bayesian}, hyperparameter tuning of machine learning algorithms \citep{snoek2012practical}, chemical design \citep{griffiths2020constrained}, and culture media optimization for cellular agriculture \citep{cosenza2022multi}. In many problems, it is not possible to observe (potentially noisy) objective values directly. Instead, a decision-maker (DM) provides preference feedback, often in the form of pairwise comparisons between options shown. This arises in applications such as animation design \citep{brochu2010bayesian}, where a DM is shown two different images and chooses the one with better characteristics (e.g. realism or resemblance to a target image); and exoeskeleton gait design \citep{tucker2020preference}, where a DM assisted by an exoskeleton walks for a short period of time using two different gait configurations and indicates the one that resulted in more comfortable walking. Preferential Bayesian optimization (PBO) \citep{brochu2010bayesian, gonzalez2017preferential}, a subframework within BO, has emerged as a powerful tool for tackling such problems.

As in standard BO, a PBO algorithm consists of two main components: a probabilistic surrogate model of the DM's latent utility function; and an acquisition function (AF), computed from the probabilistic surrogate model, whose value at a given set of $q$ alternatives quantifies the benefit of DM feedback about their preferred alternative in the set.
Several AFs for PBO have been proposed \citep{brochu2010bayesian,gonzalez2017preferential,benavoli2021preferential,siivola2021preferential,nguyen2021top}. However, most are derived from heuristic arguments and lack a proper decision-theoretic or information-theoretic justification. For example, \cite{brochu2010bayesian} selects the point that maximizes the posterior mean of the model over points in previous queries as the first alternative, and the point that maximizes the expected improvement with respect to the posterior mean value of the first point as the second alternative. Other works simply adopt AFs from the standard BO literature \citep{siivola2021preferential}, ignoring the fact that preference feedback is observed rather than direct utility values.

To address the shortcomings of existing approaches, we study the \textit{expected utility of the best option (qEUBO)}, which generalizes the EUBO AF proposed by \cite{lin2022bope} for a different problem setting, as a novel AF for PBO with a proper decision-theoretic justification. 


\paragraph{Contributions} Our contributions are as follows:
\begin{itemize}
    \item We propose qEUBO, an AF for PBO. qEUBO has a sound decision-theoretic interpretation, is simple to compute, and exhibits strong empirical performance.
    \item We show that qEUBO outperforms the state-of-the-art AFs for PBO in several synthetic and realistic test problems. Moreover, we show that qEUBO's closest competitor performs well in early iterations because it is \textit{similar} to qEUBO but its performance degrades as the number of queries grows.
    \item We show that, under sufficient regularity conditions, qEUBO's Bayesian simple regret converges to zero at a rate $o(1/n)$ as the number of queries, $n$, goes to infinity. Moreover, we show there exist problem instances where qEI, a popular acquisition from the standard BO setting that is often used in the PBO setting, has Bayesian simple regret bounded {\it below} by a strictly positive constant.
    \item We demonstrate significant benefit of asking queries with more than two alternatives. This contrasts with previous work by \citet{siivola2021preferential}, which concluded that $q>2$ only provides limited performance improvement.
\end{itemize}

\begin{figure}[ht]
\centering
\includegraphics[width=0.38\textwidth]{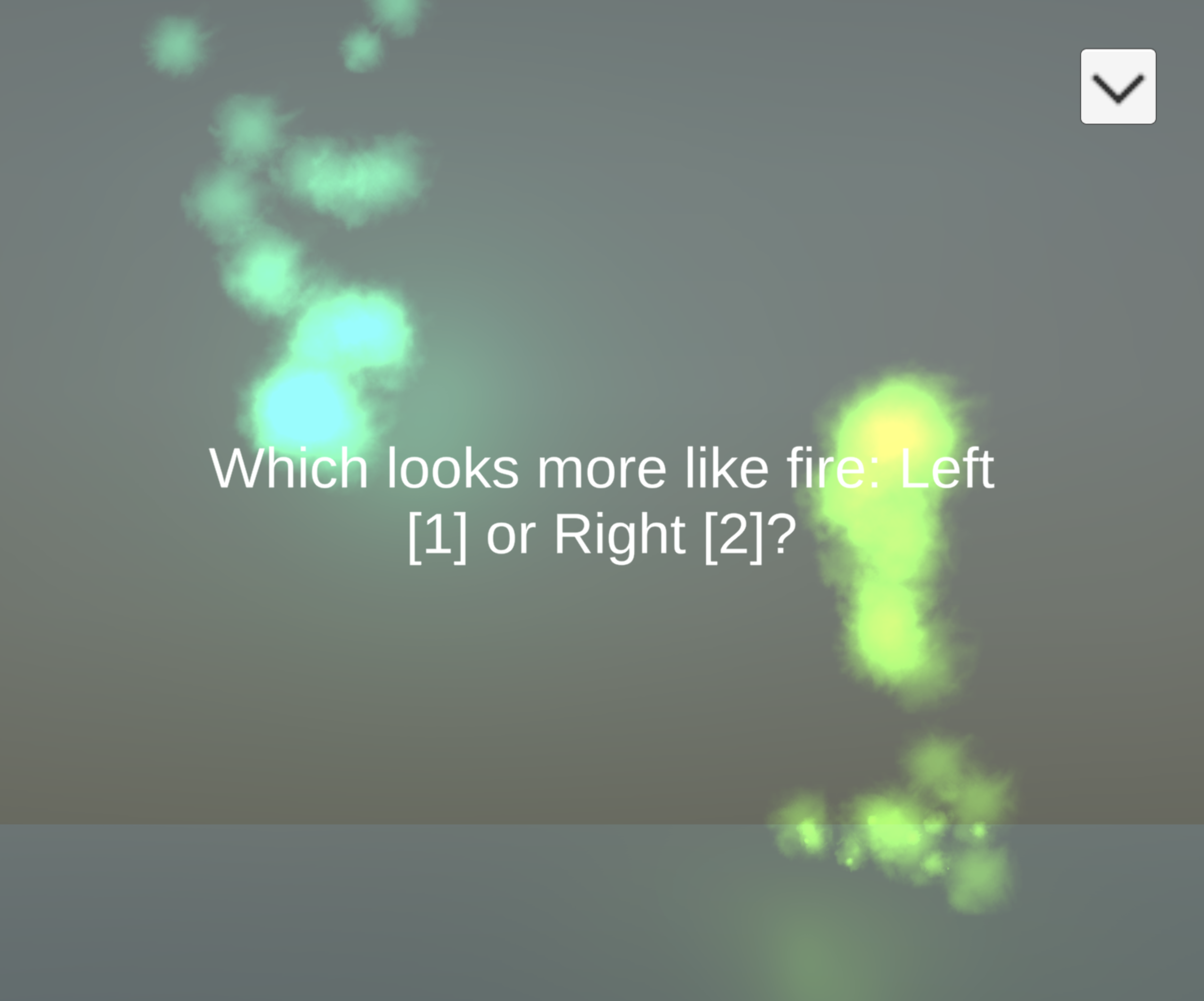}
\includegraphics[width=0.38\textwidth]{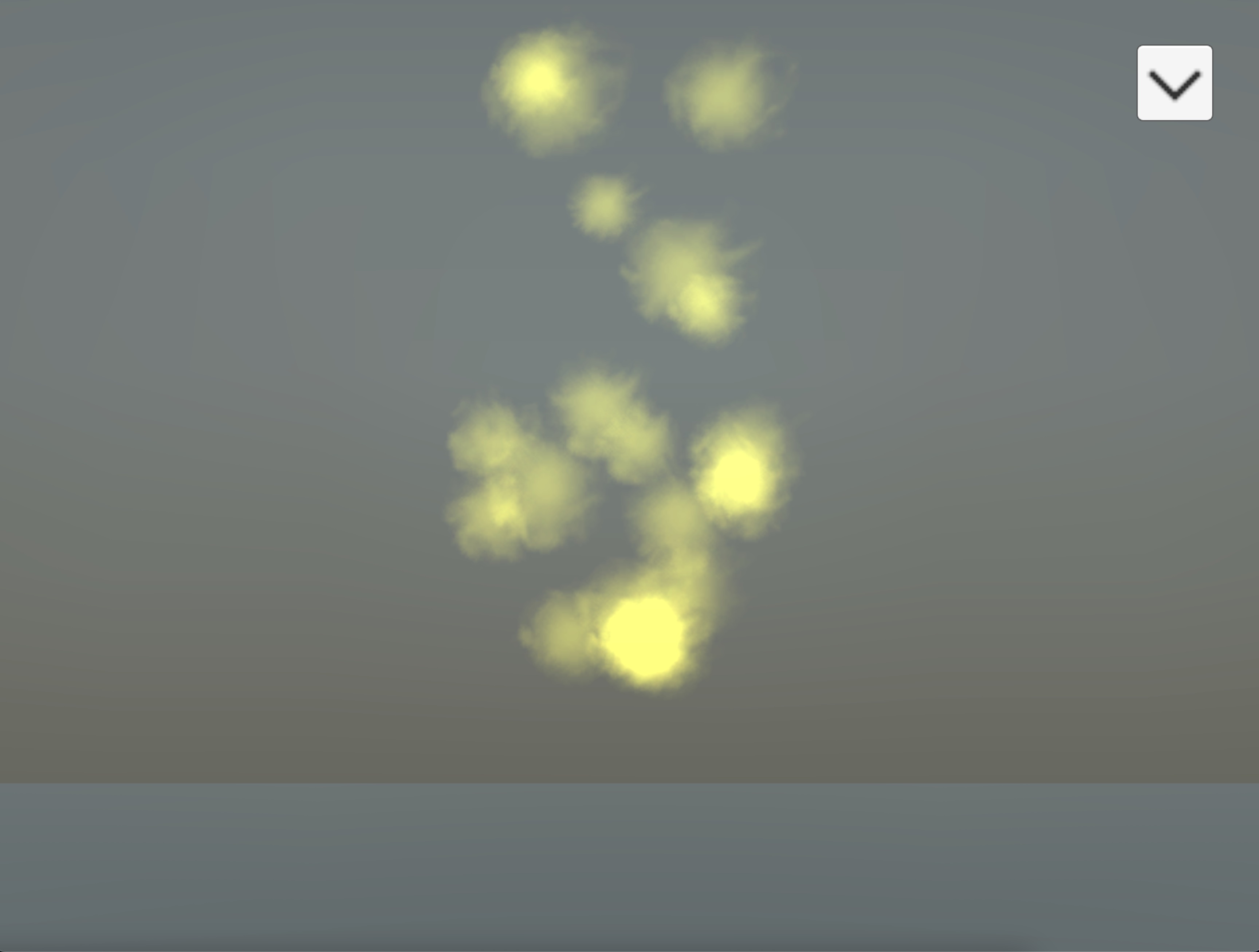}
 \caption{Fire particle rendering problem from Section~\ref{sec:experiments}, in which a human user is asked which of two animations looks more like fire (top). Final rendering results based on fitting a support vector machine model to 100 comparisons between random particle effects and then optimizing the predicted latent decision function over animation parameters (bottom).
 \label{fig:particle_demo}}
\end{figure}

\begin{figure*}[ht]
\centering
\begin{tabular}[b]{c}%
\includegraphics[width=0.33\textwidth]{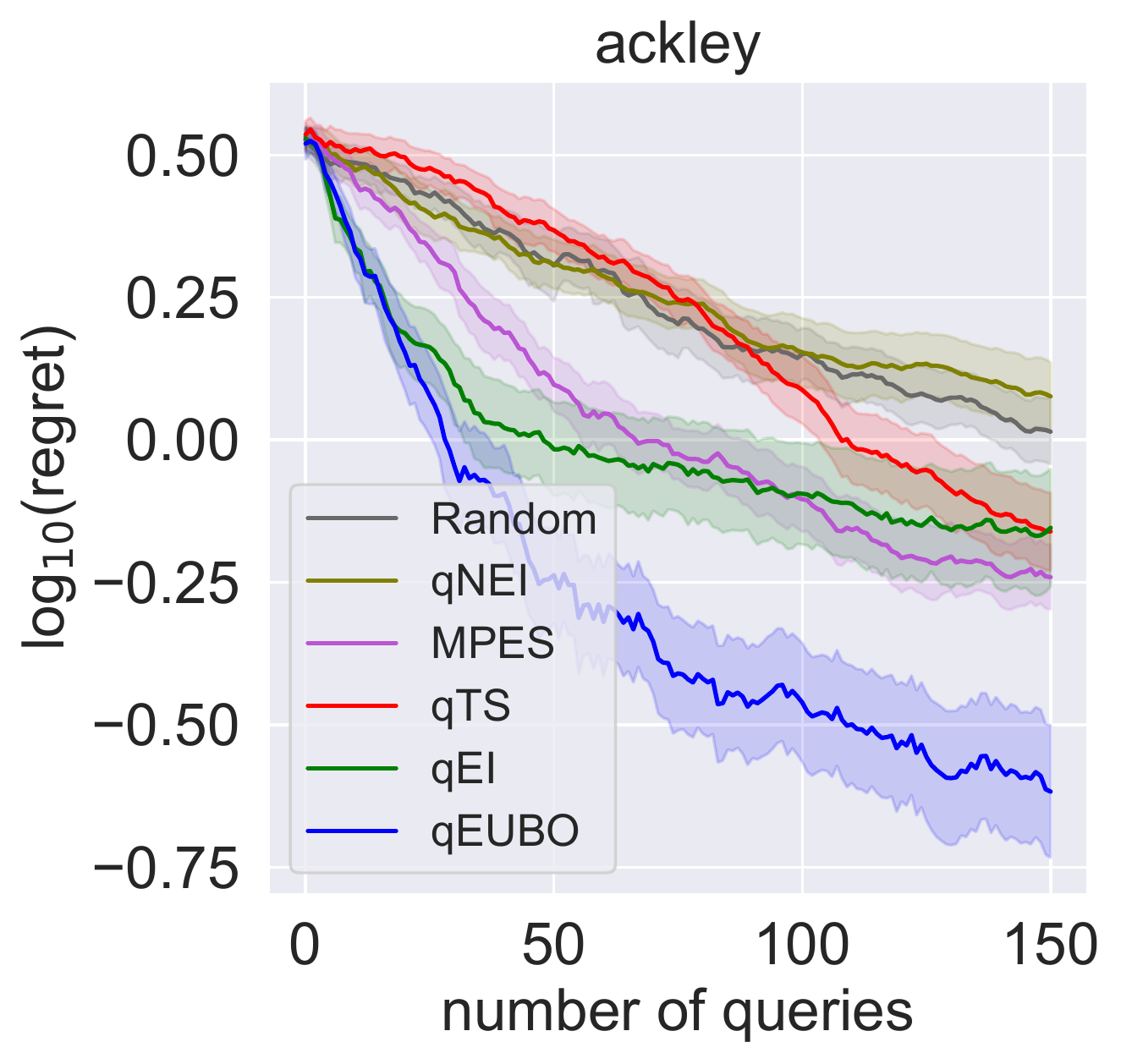}
  \includegraphics[width=0.312\textwidth]{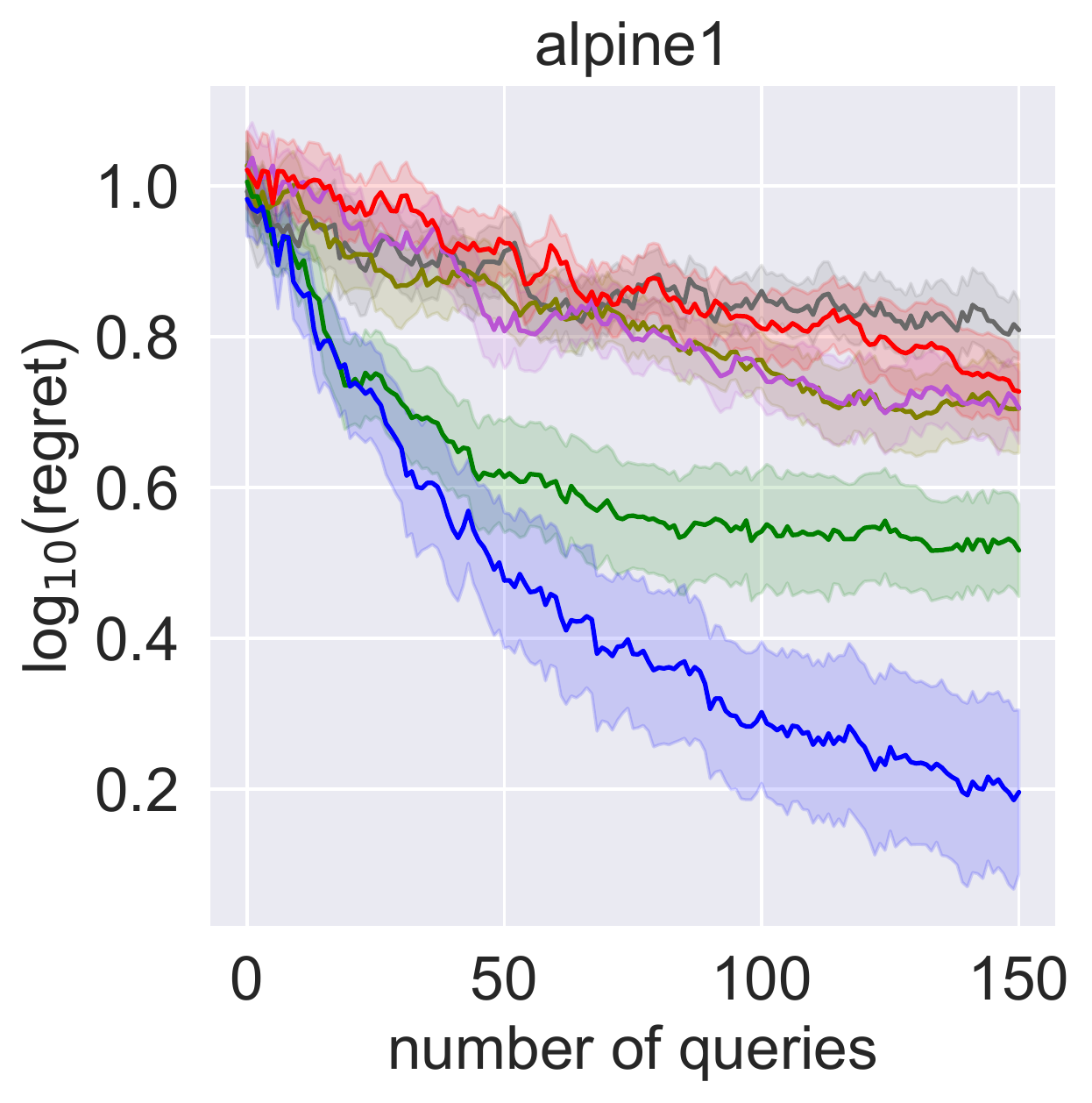}
  \includegraphics[width=0.321\textwidth]{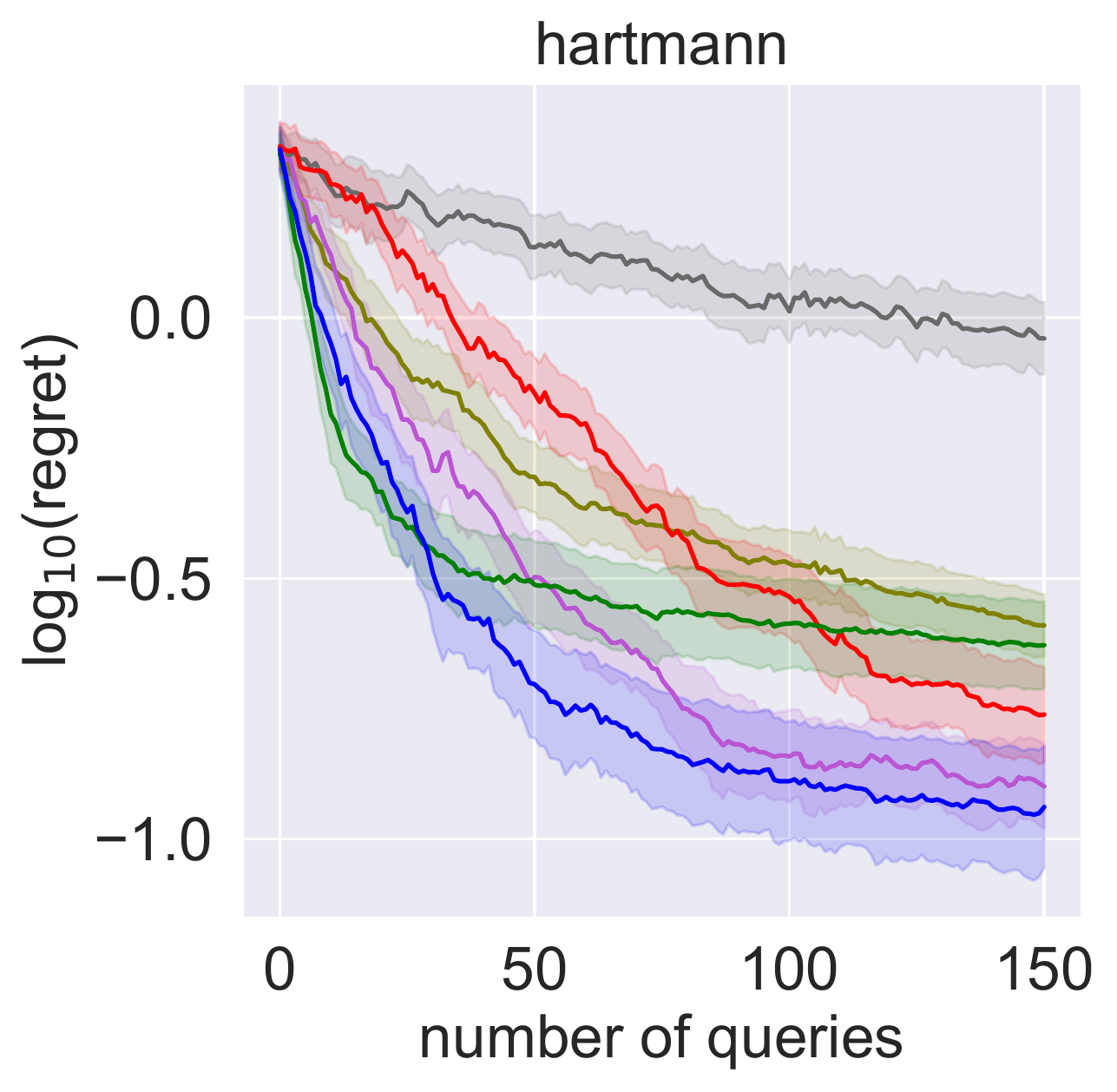}\\
  \includegraphics[width=0.33\textwidth]{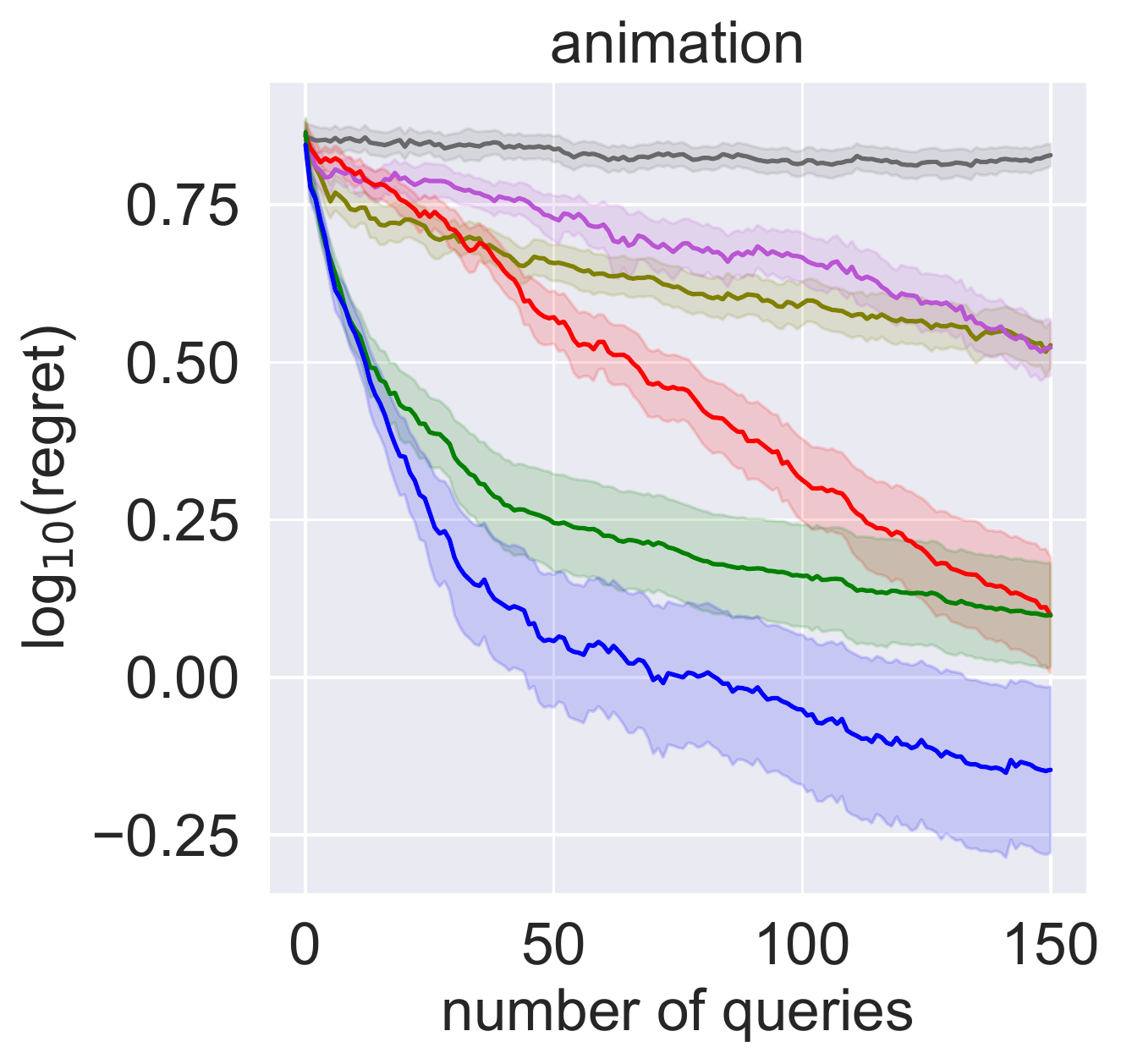}
\includegraphics[width=0.321\textwidth]{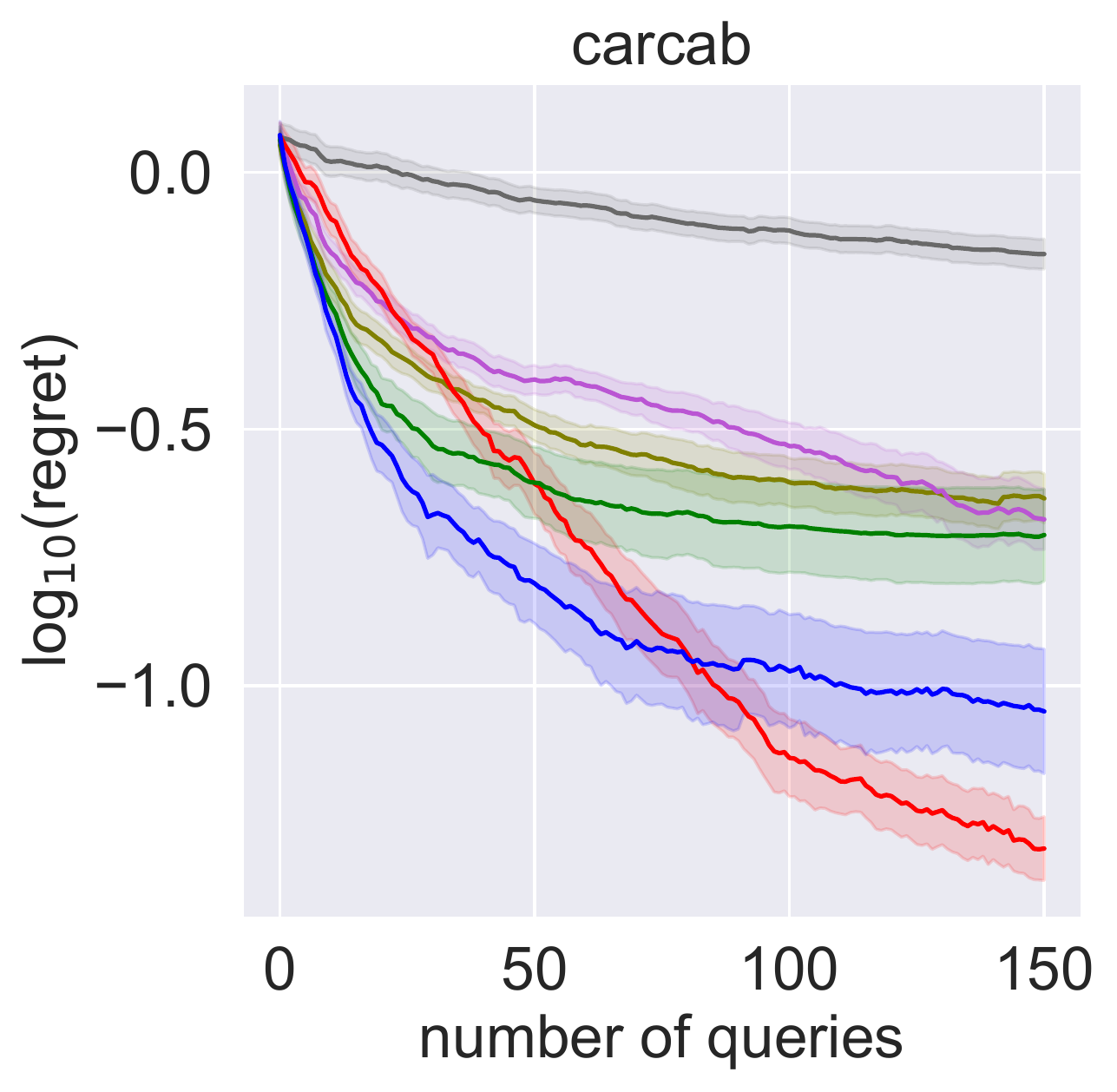}
  \includegraphics[width=0.322\textwidth]{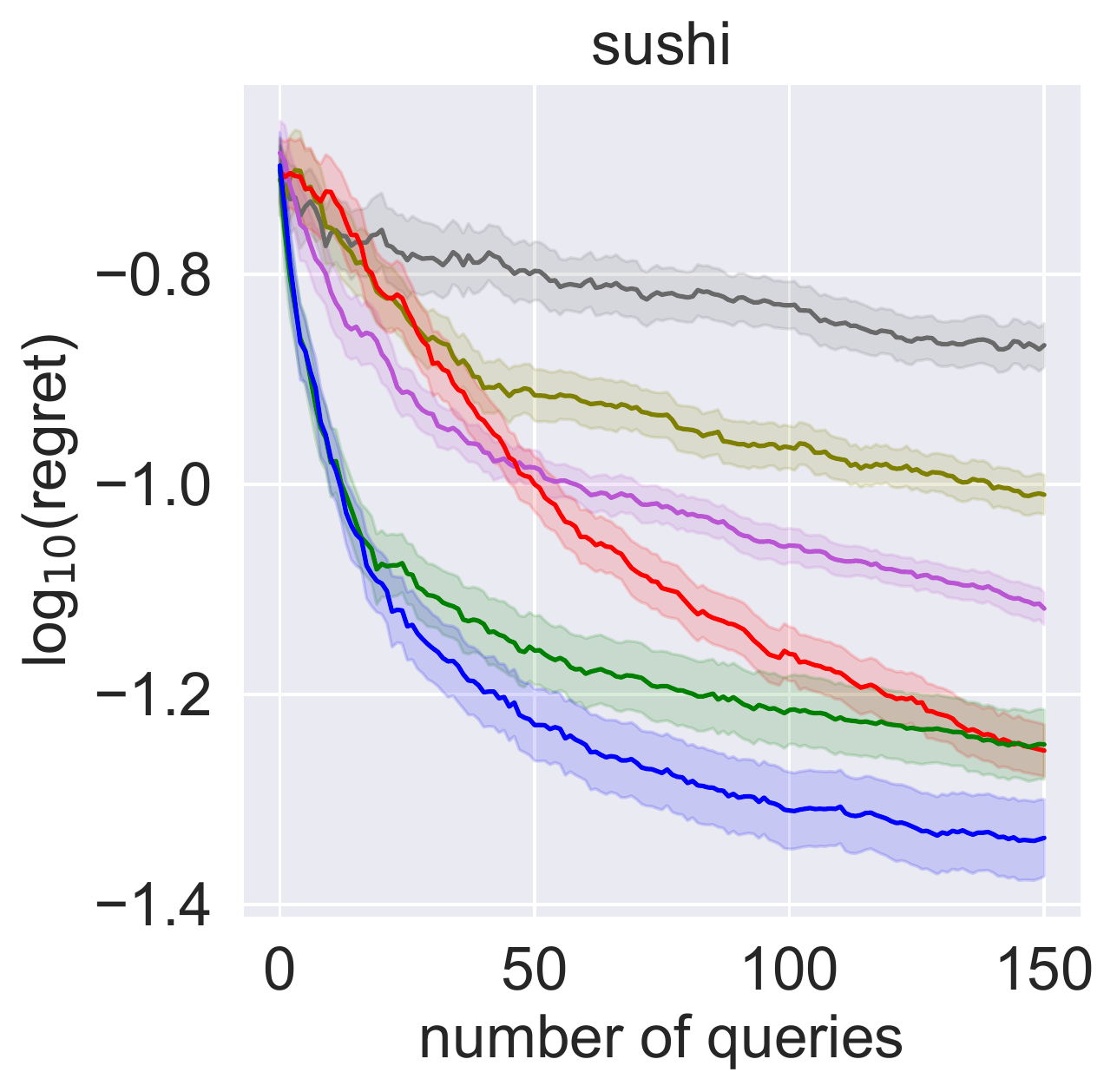}
 \end{tabular}
 \caption{log10(optimum value - utility value at the maximizer of the posterior mean) using moderate logistic noise and $q=2$ comparisons per DM query. All algorithms are shown up to 150 queries. qEUBO outperforms other algorithms on all but one problem.
 \label{fig:lr_results}}
\end{figure*}

\section{RELATED WORK}
Several AFs for PBO have been proposed in the literature. Most of them are designed via heuristic arguments \citep{brochu2010bayesian} or simply reused from the standard BO setting \citep{siivola2021preferential}. For example, \cite{brochu2010bayesian} selects the point that maximizes the posterior mean of the model over points in previous queries as the first alternative, and the point that maximizes the expected improvement with respect to the posterior mean value of the first point as the second alternative. \cite{nielsen2014perception} proposes to use the point preferred by the user in the previous query as the first alternative, and the point that maximizes the expected improvement with respect to this point as the second alternative. For $q=2$, qEUBO recovers this AF if we force the first alternative to be equal to the point preferred by the user in the previous query and optimize only over the second alternative. \cite{gonzalez2017preferential} proposes a pure exploration sampling policy along with two AFs based on the expected improvement and Thompson sampling AFs that aim to maximize the soft-Copeland's score. However, the computation of this score requires integration over the optimization domain, thus making these algorithms intractable even for problems of moderate dimension. \cite{siivola2021preferential} proposes using \textit{batch} versions of the expected improvement and Thompson sampling AFs from standard BO for selecting the points in each query. Since utility values are not observed directly, the batch expected improvement is adapted by defining the improvement with respect to the maximum posterior mean value over points in previous queries, along the lines of the approach followed by \cite{brochu2010bayesian}. Batch Thompson sampling is defined as in the standard BO setting: each point in the batch is selected as the point that maximizes an independently drawn sample from the utility's posterior distribution. \cite{nguyen2021top} proposes the multinomial predictive entropy search (MPES) AF for top-$k$ ranking BO, a slightly more general framework where the DM selects her most $k$ preferred alternatives in each query. MPES selects the query that  maximizes
the information gain on the utility
function's maximizer through observing the DM's feedback. It can be seen as a principled adaptation of the predictive entropy search (PES) AF for standard BO \citep{hernandez2014predictive}. Like with PES, the computation of MPES requires approximating an intractable multi-dimensional integral with respect to the posterior distribution on the utility function's maximizer. This is computationally challenging, especially in problems of moderate dimension, and inaccurate approximation can lead to a degraded performance. To our knowledge, the AFs proposed by \cite{siivola2021preferential} and \cite{nguyen2021top} are the only existing ones allowing for queries with more than two alternatives. Finally, \cite{benavoli2021preferential} proposes an AF where the first alternative is the point chosen by the user in the previous query, and the second one is obtained by maximizing a linear combination between the logarithm of the probability of improvement and the information gain; the weight of this linear combination is a hyperparameter of the algorithm. It also proposes two other AFs based on Thompson sampling and GP-UCB \citep{srinivas2012information}.

The above AFs (except MPES) were derived via heuristic arguments. In contrast, qEUBO is derived following a principled decision-theoretic analysis modeling the fact that, in PBO, observations are comparisons instead of direct utility values.
qEUBO's approach is consistent with the rigorous decision-theoretic or information-theoretic analysis used to derive principled AFs in standard BO. 
Moreover, unlike MPES, qEUBO is easy to compute and comes with a convergence guarantee. Finally,  qEUBO outperforms MPES significantly in our empirical evaluation.

qEUBO, restricted to the case $q=2$, was first discussed by \cite{lin2022bope} in the context of preference exploration for multi-attribute BO. In this context, the DM does not express preferences directly over alternatives but over attributes of these alternatives, which are assumed to be time-consuming to evaluate. As a consequence, qEUBO is not used directly to find alternatives to show to the DM. Instead, it is combined with a probabilistic surrogate model of the mapping from alternatives to attributes to find hypothetical attribute vectors over which the DM expresses preferences. Our work places qEUBO in the context of PBO and extends its definition to queries with $q > 2$ alternatives. We also  generalize the analysis of \cite{lin2022bope} by showing that maximizing qEUBO recovers a one-step optimal solution when responses are noise-free for $q > 2$. Finally, we provide a novel convergence analysis for qEUBO.

The connection between qEUBO and the one-step Bayes optimal policy relates qEUBO to the knowledge gradient class of sampling policies for sequential data collection, which are, by definition, one-step Bayes optimal \citep{frazier2008knowledge}. Knowledge gradient AFs have been widely used in standard BO \citep{wu2016parallel,wu2017bayesian,cakmak2020bayesian} and are known for their superior performance over simpler AFs such as expected improvement or Thompson sampling, especially when observations are noisy or the objective function is highly multi-modal \citep{wu2016parallel,balandat2020botorch} or when these simpler AFs are used in settings where they lack a meaningful interpretation. At the same time, they are typically very challenging to maximize \citep{balandat2020botorch}, often resulting in high computation times and degraded performance in problems of moderate dimension. The former is particularly problematic in the PBO setting where queries are often required to be generated in real time. Since qEUBO is significantly simpler to compute, we effectively overcome the computational burden commonly faced by knowledge gradient AFs. While this equivalence does not hold anymore when responses are noisy, we show that qEUBO still enjoys an additive constant approximation guarantee to the one-step Bayes optimal policy.

Our work is also related to the literature on dueling-bandits \citep{yue2012k,bengs2021preference}. Like in our problem setting, the DM is assumed to express preference feedback over sets (typically pairs) of alternatives. However, most of these approaches assume a finite number of alternatives and often also independence across pairs of alternatives. The double Thomson sampling strategy for dueling bandits proposed by \cite{wu2016double} is analogous to the Thompson sampling AF for PBO proposed by \citet{siivola2021preferential}. 

Finally, our work is also related to a broader stream of research on computational preference elicitation \citep{braziunas2006computational}. This work focuses on problems with a finite set of alternatives, where it suffices to estimate a ranking of the alternatives, or on estimating the underlying DM's utility function within a parametric class of functions. In particular, our work is closely related to \cite{viappiani2010optimal}, which derived an analogous result relating optimal recommendation sets with one-step Bayes optimal query sets. However,  \cite{viappiani2010optimal} proposes using optimal recommendation sets to query the DM, which differ from those queries selected by qEUBO when responses are noisy. As a consequence, our analysis under noisy responses also focuses on relating qEUBO to the one-step Bayes optimal policy rather than the policy that selects optimal recommendation sets.

\begin{figure*}[ht]
\centering
\begin{tabular}[b]{c}%
\includegraphics[width=0.32\textwidth]{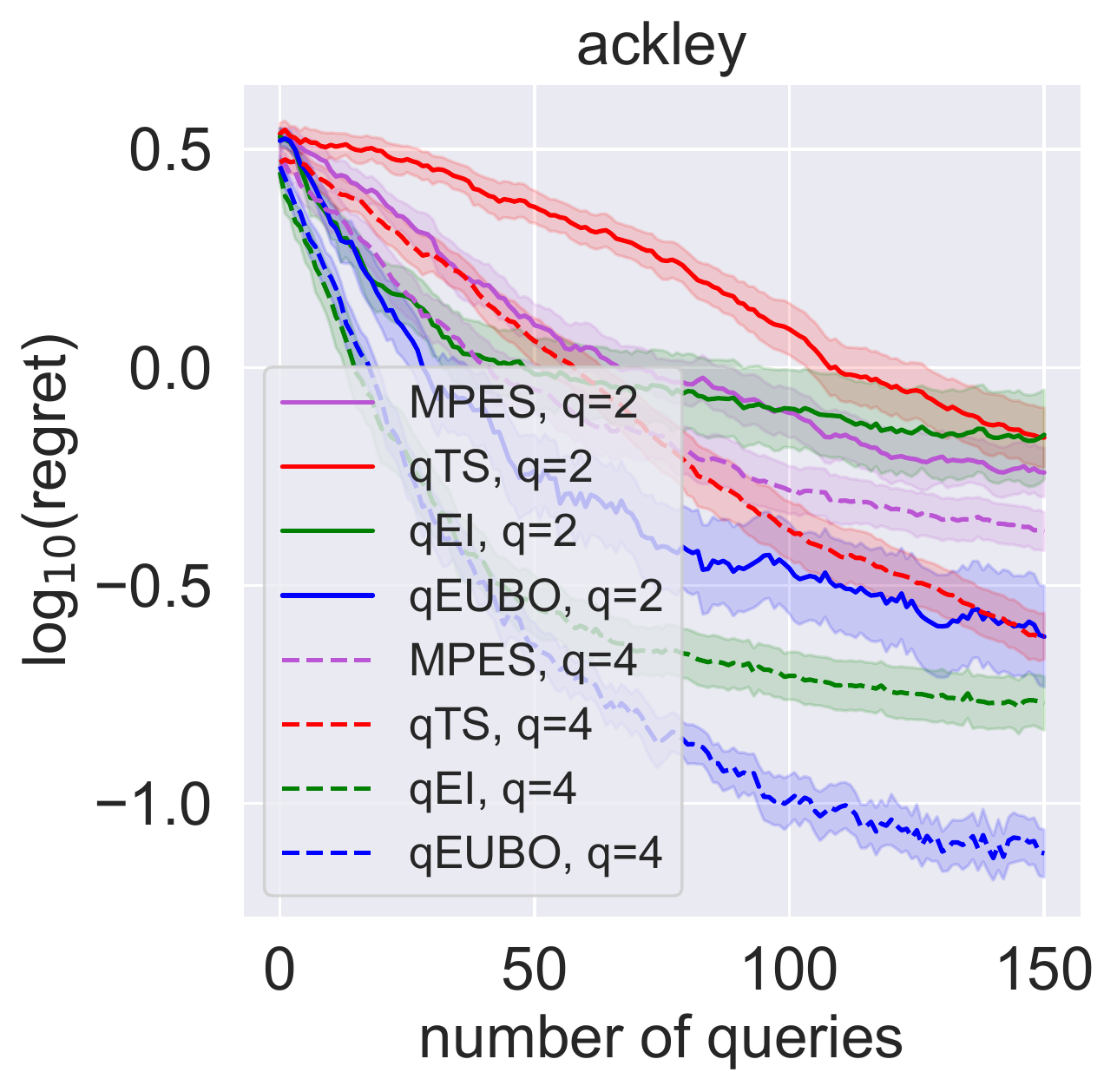}
  \includegraphics[width=0.32\textwidth]{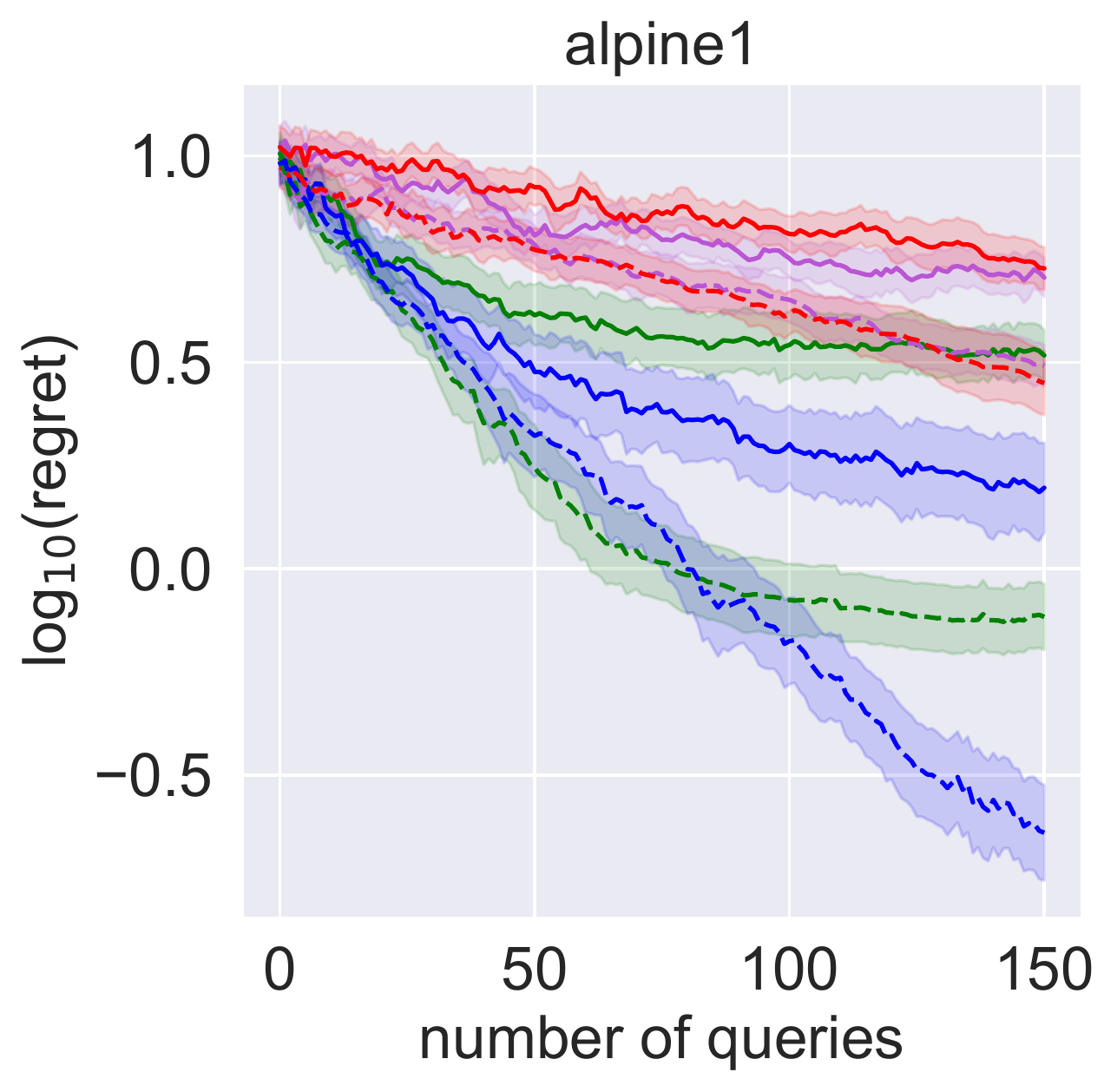}
  \includegraphics[width=0.32\textwidth]{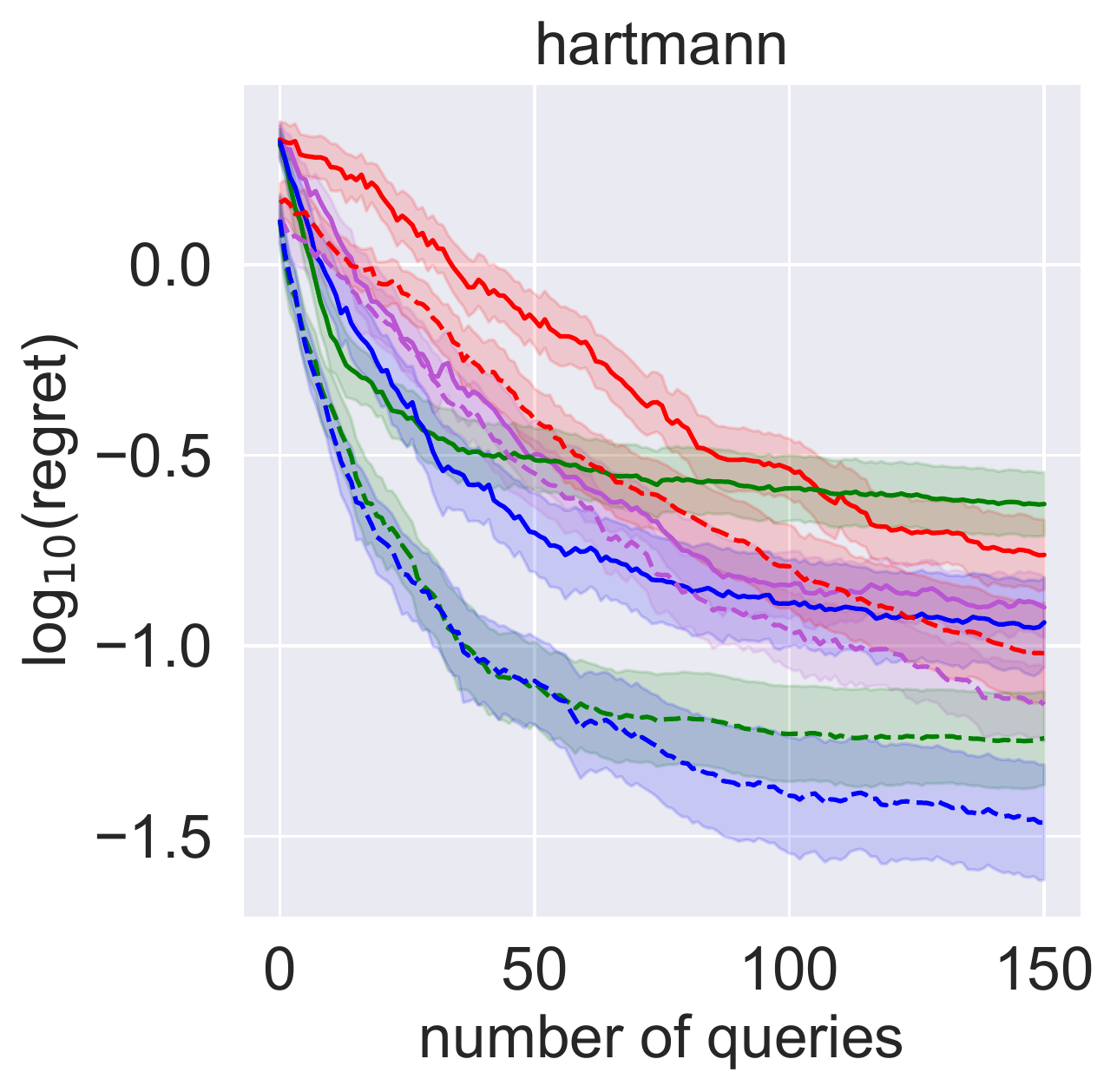}
  \\
  \includegraphics[width=0.32\textwidth]{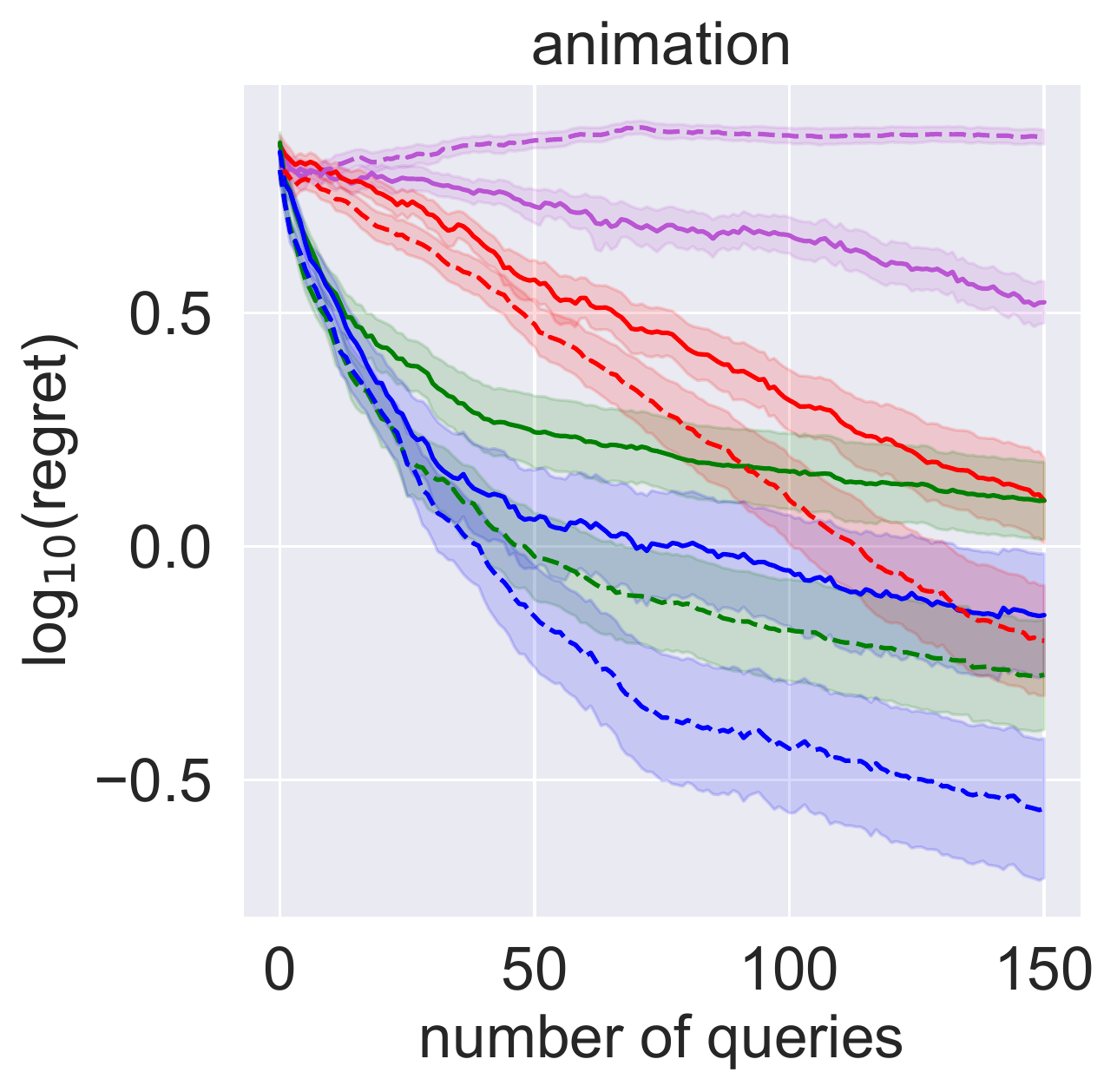}
\includegraphics[width=0.32\textwidth]{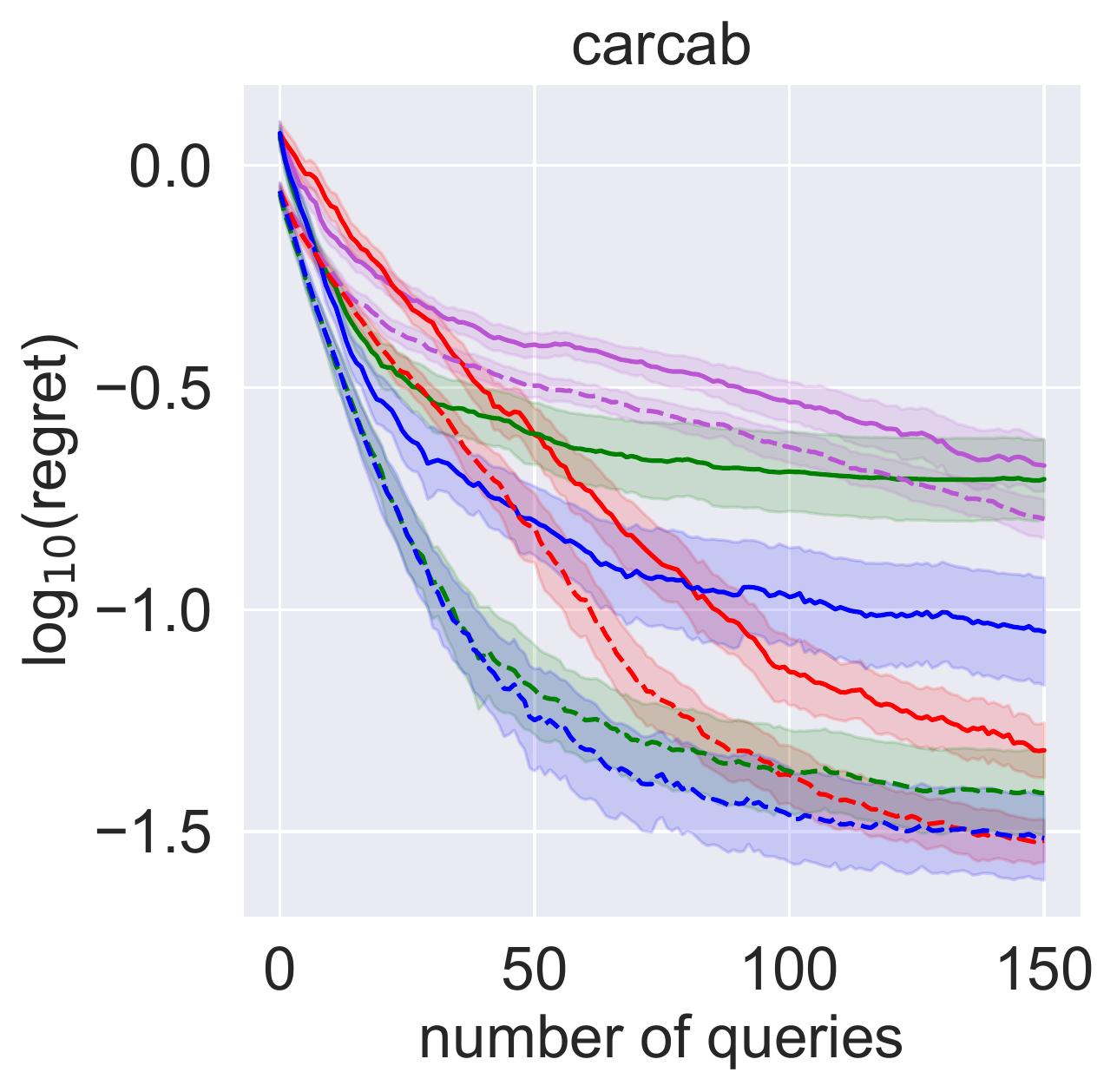}
  \includegraphics[width=0.32\textwidth]{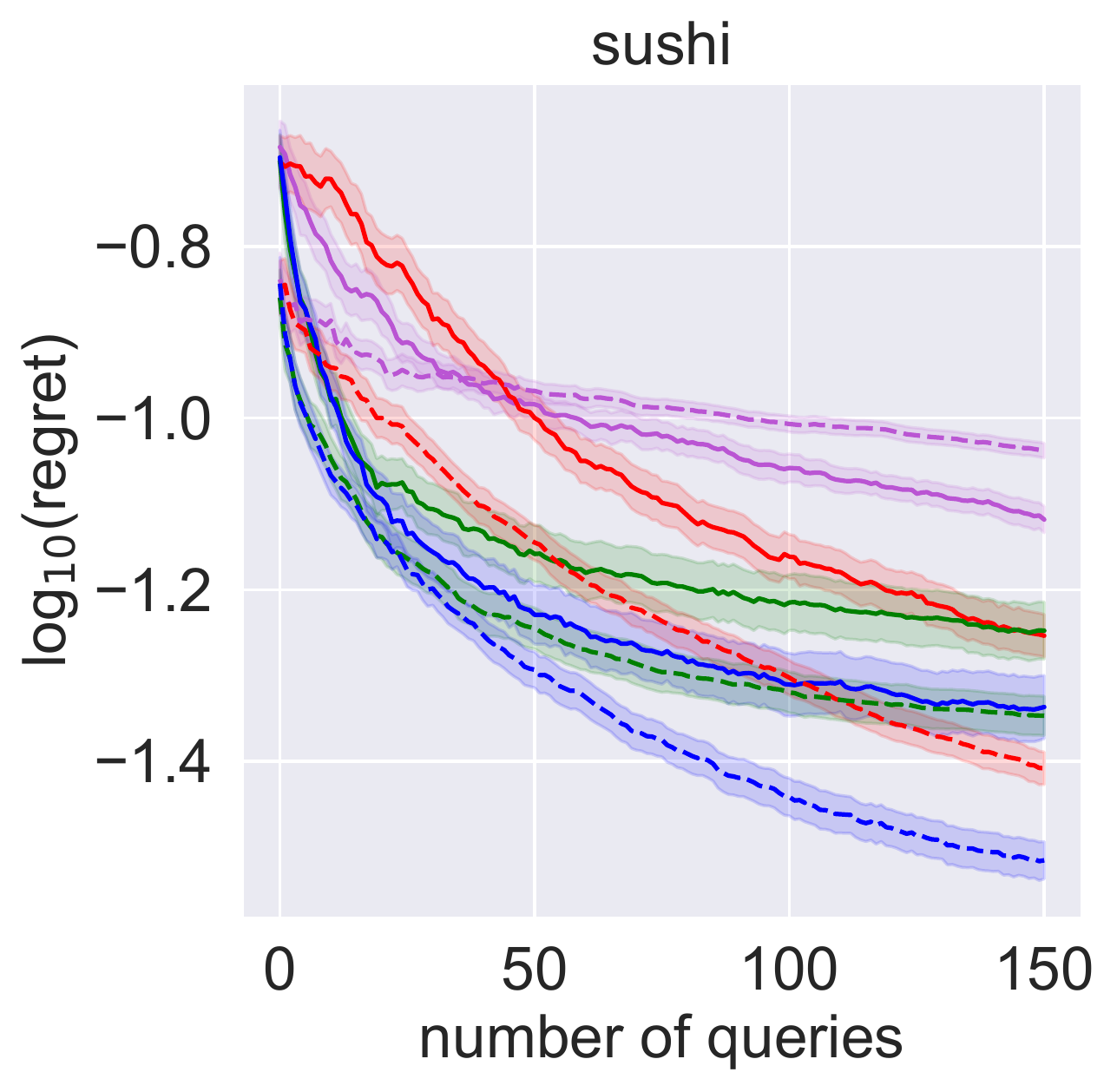}
 
 \end{tabular}
 \caption{log10(optimum value - utility value at the maximizer of the posterior mean) for $q\in \{2,4\}$. Algorithms are shown up to 150 queries. qEUBO outperforms all other algorithms on all problems. Including more alternatives per query ($q=4$) allows regret to decline more quickly.
 \label{fig:lr_batch_results}}
\end{figure*}

\section{PROBLEM SETTING}
We denote the space of \textit{alternatives} or \textit{options} by $\X$. Succinctly, our goal is to find the best possible alternative in $\X$ according to the DM's underlying preferences. These preferences are encoded via a latent utility function, $f:\X\rightarrow\R$. We model $f$ through a general Bayesian prior distribution. In our experiments we use a Gaussian process (GP) prior, but our derivation and analysis of qEUBO does not make this assumption and is applicable to more general priors. 

 At every interaction with the DM, an algorithm selects a \textit{query}, $X = (x_1,\ldots, x_q) \in \X^q$. The DM then expresses her most preferred alternative among these $q$ points. This response is denoted by $r(X)\in\{1, \ldots, q\}$, where $r(X)=i$ if $x_i$ is the alternative chosen by the DM. The DM's responses may be not be always consistent with the underlying utility function. We model this via a parametric likelihood function $L( \ \cdot \ ; \lambda): \R^q \rightarrow \R^q$ such that
 \begin{equation*}
    \prob(r(X)=i\mid f(X)) = L_i(f(X);\lambda),
\end{equation*}
where $L_i(f(X);\lambda)$ is the $i$-th component of $L(f(X) ; \lambda)$ and $\lambda$ is estimated along with other parameters of the model. Our numerical experiments and Theorem~\ref{thm:2} assume a logistic likelihood function of the form
\begin{equation*}
   L_i(f(X);\lambda) = \frac{\exp(f(x_i)/\lambda)}{\sum_{j=1}^q \exp(f(x_j)/\lambda)},
\end{equation*}
for $i = 1,\ldots, q$, where $\lambda \geq 0$ is the \textit{noise level} parameter. For $\lambda = 0$, the above expression is defined as its right-hand limit as $\lambda$ converges to $0$. It can be easily shown that $\lambda=0$ recovers a noise-free response likelihood. Theorem~\ref{thm:eubo_conv}  allows for a broader class of likelihood functions. Details are provided in Section~\ref{sec:thm3_4}.

Let $\D^{(n)} = \{\left(X_m, r(X_m)\right)\}_{m=1}^n$ denote the data collected after $n$ queries and $\E_n$ denote the conditional expectation given $\D^{(n)}$. Following the decision-theory literature, if we decide to stop at time $N$, we will recommend the point that maximizes the DM's expected utility given the data collected so far; i.e., an element of $\argmax_{x\in\X}\E_N[f(x)]$. Thus, we wish to select the queries $X_1, \ldots, X_N$ so that the expected utility received by the DM under our recommendation, $\max_{x\in\X} \E_N[f(x)]$, is as large as possible.

\section{qEUBO}
\subsection{qEUBO and the One-Step Bayes Optimal Policy}
To motivate our AF, we begin by discussing the one-step Bayes optimal policy, i.e., the policy that chooses at every iteration the query that would be optimal if it were the last one.  To this end, we define for an arbitrary query $X\in\X^q$,
\begin{equation*}
    V_n(X) = \E_n\left[\max_{x\in\X}\E_{n+1}[f(x)]\mid X_{n+1} = X\right].
\end{equation*}
This is the expected utility received by the DM if one last additional query $X_{n+1} = X$ is performed. The one-step Bayes optimal policy chooses at every iteration the query that maximizes $V_n$.

Since $\max_{x\in\X}\E_n[f(x)]$ does not depend on $X_{n+1}$, maximizing $V_n$ is equivalent to maximizing
\begin{equation*}
\E_n\left[\max_{x\in\X}\E_{n+1}[f(x)] - \max_{x\in\X}\E_n[f(x)]\mid X_{n+1} = X\right].
\end{equation*}
The above expression is analogous to the knowledge gradient AF from standard BO \citep{frazier2009knowledge, scott2011correlated, wu2016parallel}. As mentioned earlier, knowledge gradient AFs often outperform simpler AFs. However they are also very challenging to maximize due to their nested expectation-maximization structure.

Our main result shows that, when the DM responses are noise-free, maximizing $V_n$ is equivalent to maximizing a simpler AF. We define the \textit{expected utility of the best option (qEUBO)} AF by
\begin{equation*}
    \eubo_n(X) = \E_n\left[\max\{f(x_1),\ldots, f(x_q)\}\right].
\end{equation*}
Under this definition, the following result holds.
\begin{theorem}
\label{thm:1}
Suppose the DM's responses are noise-free. Then,
\begin{equation*}
 \argmax_{X\in\X^q}\eubo_n(X) \subseteq \argmax_{X\in\X^q}V_n(X).   
\end{equation*}
\end{theorem}
Thus, to find a maximizer of $V_n$, it suffices to maximize $\eubo_n$. This is a significantly simpler task as it does not require solving a nested stochastic optimization problem. When the posterior over $f$ is Gaussian or approximated via a Gaussian distribution (e.g., via a Laplace approximation), $\eubo_n$ can be efficiently maximized via sample average approximation \citep{balandat2020botorch}. This is the approach we pursue in our experiments. Moreover, if $q=2$, $\eubo_n$ has a closed form expression in terms of the posterior mean and covariance functions \citep{lin2022bope}.

When the DM's responses are noisy, maximizing $\eubo_n$ is no longer equivalent to maximizing $V_n$. However, the result below shows that if noise in the DM's responses follows a logistic likelihood, maximizing $\eubo_n$ still recovers a high-quality query. Formally, we show the following.
\begin{theorem}
\label{thm:2}
Suppose that the DM's responses follow the logistic likelihood function with parameter $\lambda$ defined above. Denote $V_n$ as $V_n^\lambda$ to make its dependence on $\lambda$ explicit. If $X^*\in\argmax_{X\in\X^q}\eubo_n(X)$, then
\begin{equation*}
    V_n^\lambda(X^*) \geq \max_{X\in\X^q}V_n^0(X) - \lambda C,
\end{equation*}
where $C = L_W((q-1)/e)$, and $L_W$ is the Lambert $W$ function \citep{corless1996lambertw}.
\end{theorem}
The above two results extend those shown by \cite{lin2022bope} to the logistic likelihood and $q > 2$. Their proofs can be found in Section~\ref{sec:thm1_2}.

\subsection{qEUBO and qEI}
\label{sec:qEI}

\begin{figure*}
    \centering
    \includegraphics[width=0.32\textwidth]{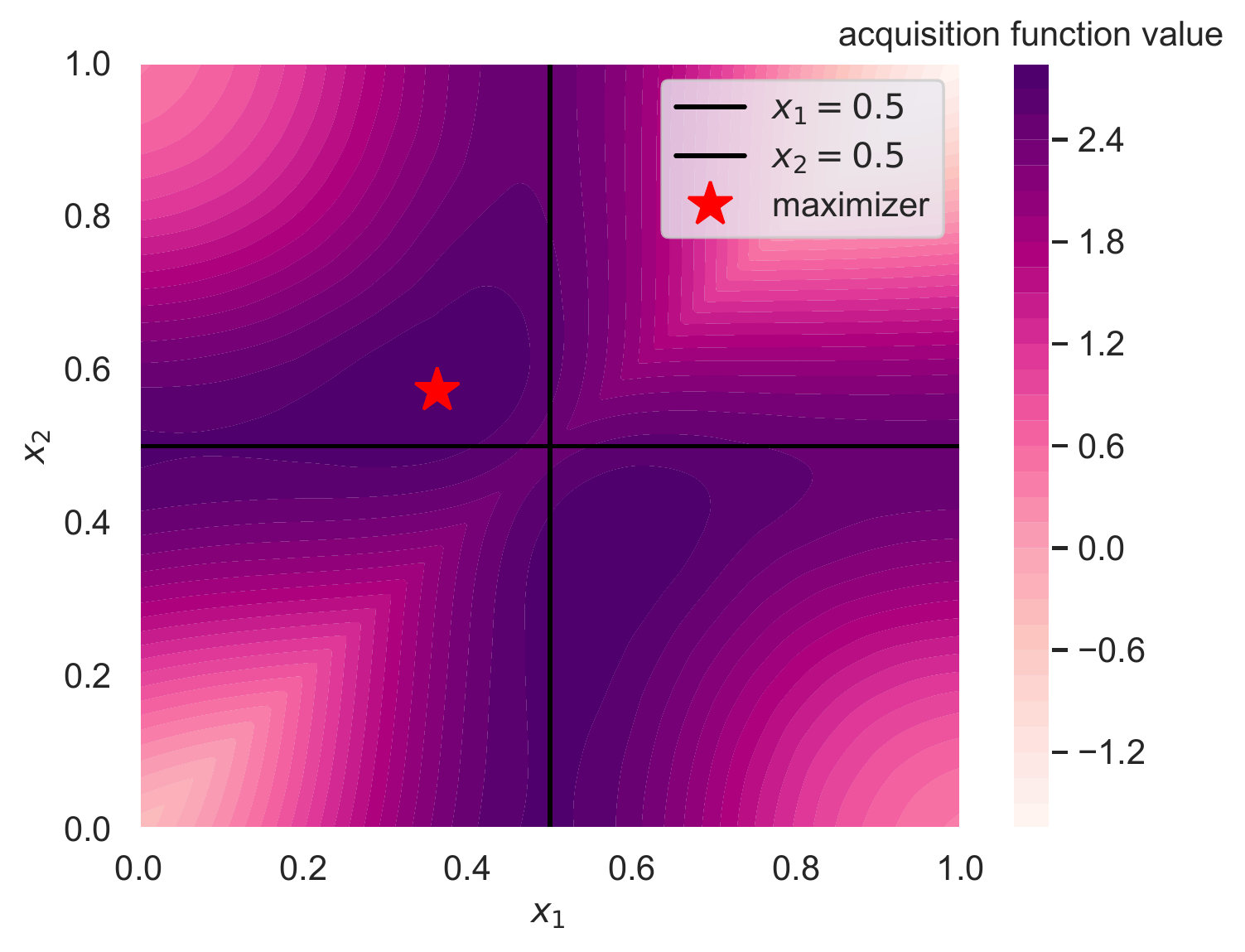}
  \includegraphics[width=0.32\textwidth]{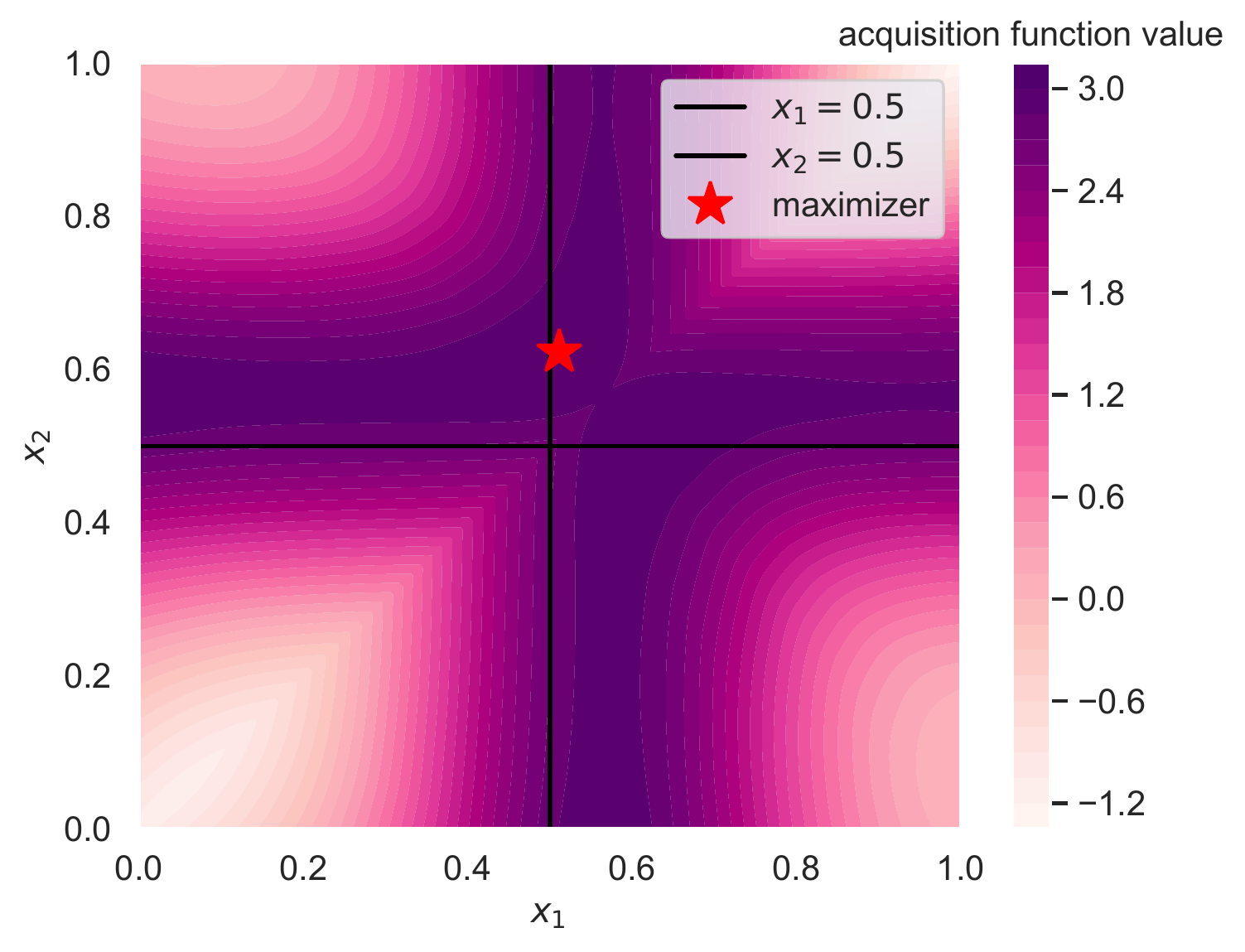}
  \includegraphics[width=0.32\textwidth]{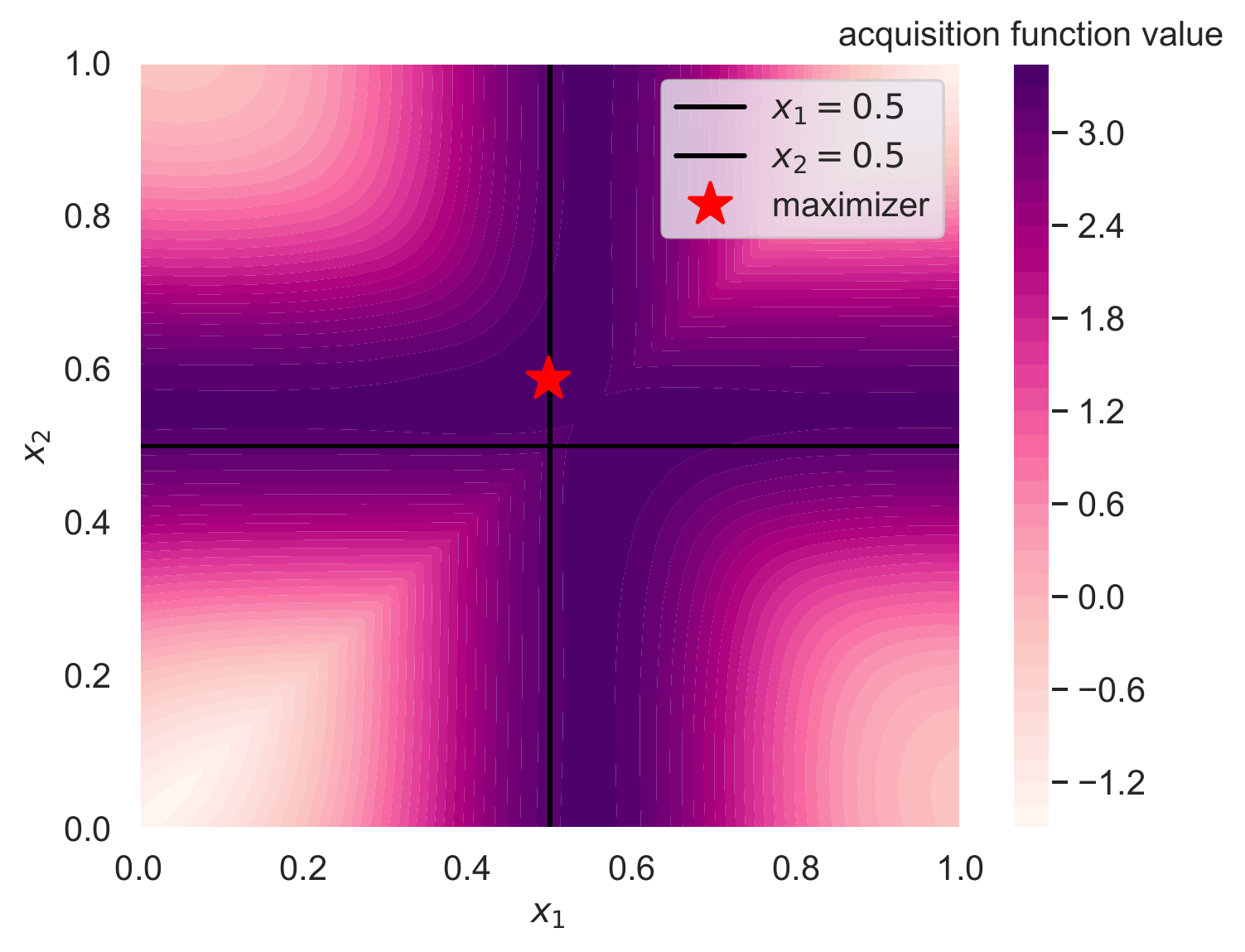}
  \\
  \includegraphics[width=0.32\textwidth]{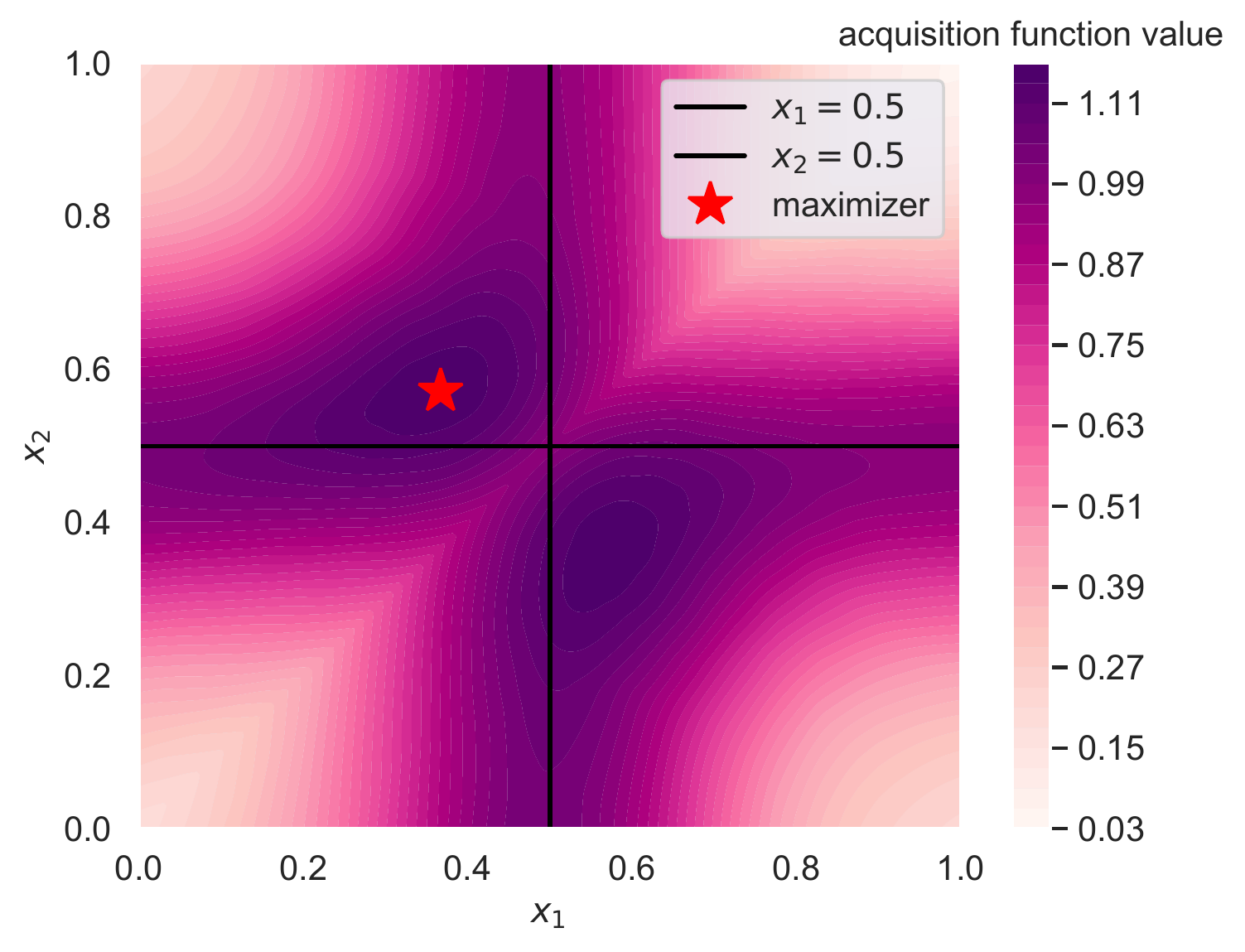}
\includegraphics[width=0.32\textwidth]{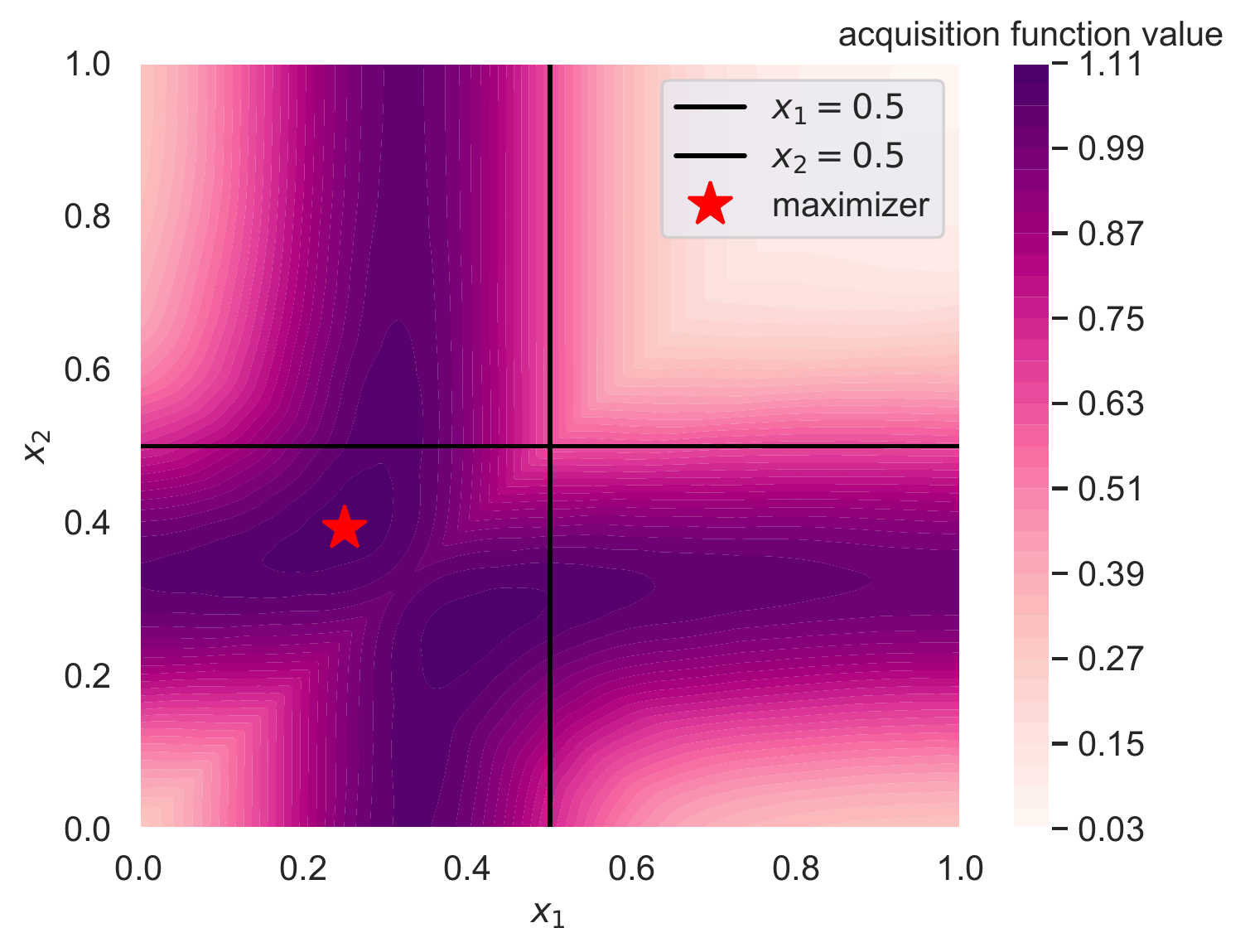}
  \includegraphics[width=0.32\textwidth]{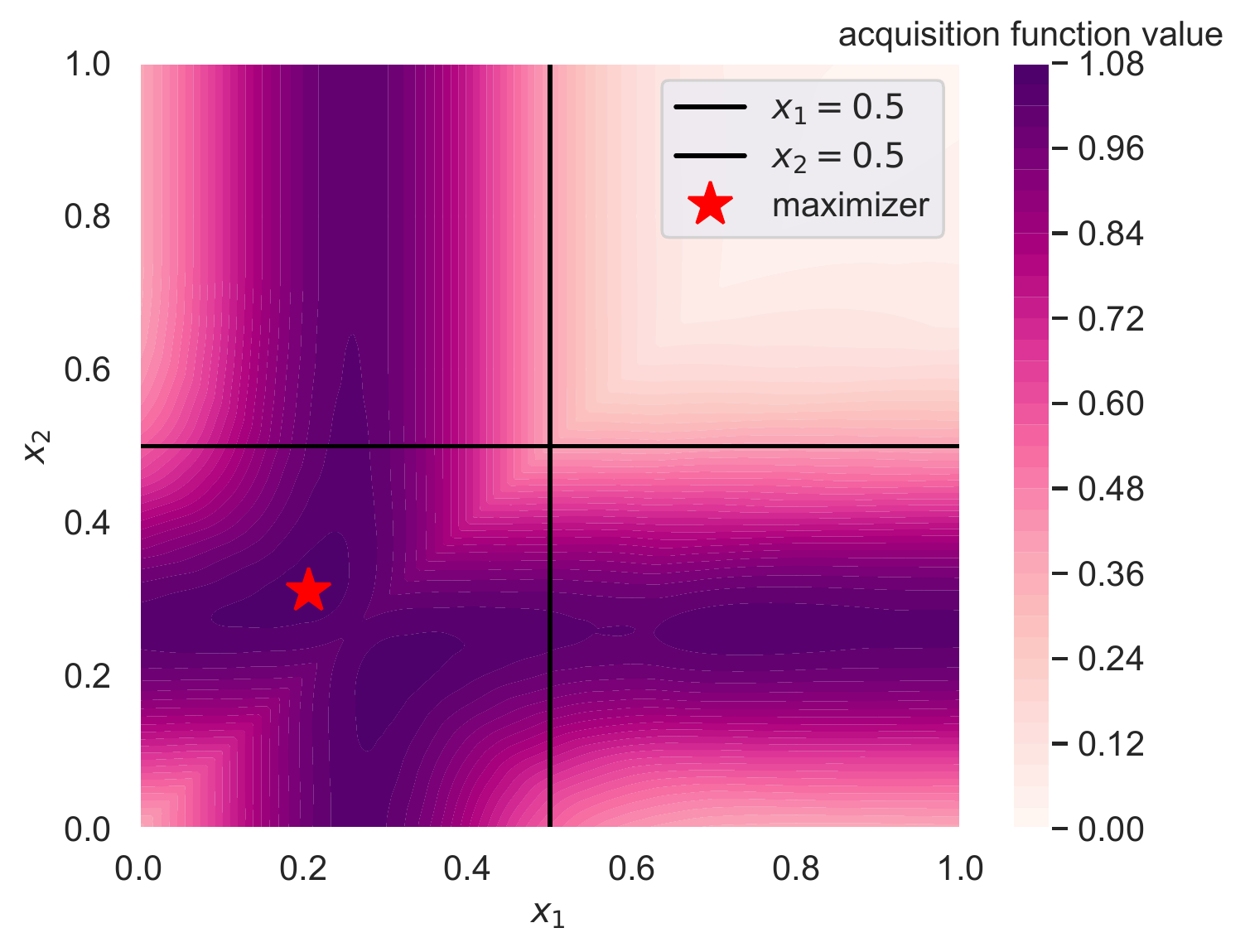}
    \caption{Comparison of qEUBO (top row) against qEI (bottom row) on a 1-dimensional problem (a quadratic function with a single maximum at $x=0.5$) with $q=2$. The first column shows the AF over the two alternatives to be included in the query after training a preferential GP on $5$ randomly generated queries. The second column shows the AF after 5 more queries, generated by the given AF. The third column shows the AF after 5 more queries generated via that row's AF.}
    \label{fig:qeubo_vs_qei}
\end{figure*}

The batch expected improvement AF, commonly known as qEI, was developed in the context of parallel BO \citep{ginsbourger2008multi,wang2016parallel}, where it enjoys a meaningful decision-theoretic interpretation. It was adapted to the PBO setting by \cite{siivola2021preferential}. While qEI lacks a meaningful interpretation in the PBO setting, it often has good performance. Here, we show that qEUBO is related to qEI. This connection sheds light on qEI's strong empirical performance as an AF for PBO. However, we also show qEI has significant drawbacks that cause it to have poor performance in some practical scenarios.  

Observe that 
$\eubo_n(X) = \E_n[F(X)]$ 
where 
$F:\X^q \rightarrow \R$ is defined by $F(X) = \max_{i=1,\ldots, q}\left\{f(x_1),\ldots, f(x_q)\right\}$. 
Moreover, if $I_n$ is any value that does not depend on $X$, then maximizing $\eubo_n$ produces the same query as maximizing $\E_n\left[F(X) - I_n\right]$. Now observe that, if $I_n= \max\left\{\E_n[f(x)] : x\in \cup_{m=1}^nX_m\right\}$, then $\qei_n(X) = \E_n[\{F(X) - I_n\}^+]$ recovers the qEI AF proposed by \cite{siivola2021preferential}. From these expressions we observe that, if $F(X)$ is typically larger than $I_n$ so that $F(X)-I_n = \{F(X)-I_n\}^+$, then optimizing $\eubo_n(X)$ should produce an optimal $X$ similar to the one obtained by optimizing $\qei_n(X)$. This is often the case in early iterations, when $I_n$ is small. However, as $I_n$ becomes larger, we expect this to occur less frequently, making qEI and qEUBO produce more different queries. Moreover, when they diverge, qEI can perform quite poorly, as we will see later. 

When the incumbent alternative (the one whose posterior mean achieves $I_n$) has low variance, as typically results from comparing a good alternative against many other alternatives, then qEI will become increasingly reluctant to include it or points near it into the next query.
In standard BO, this reluctance is appropriate because re-measuring the incumbent will not generate an improvement.
But, in PBO, there is great value in comparing an incumbent alternative to another alternative that might be better --- this is a primary way that we evaluate new alternatives.

This is also consistent with experimental results discussed later in Section~\ref{sec:experiments} and shown in Figures~\ref{fig:lr_results}, \ref{fig:lr_batch_results} and \ref{fig:baseline}: qEUBO and qEI tend to perform similarly early on when we have asked the DM few queries; later, qEI tends to stall, while qEUBO continues to reduce its simple regret.
This intuition is also codified in Theorem~\ref{thm:qEI} in the next section, which shows an example in which qEI fails to be consistent.

To illustrate this further, Figure~\ref{fig:qeubo_vs_qei} compares qEUBO and qEI on a simple 1-dimensional example problem (a quadratic function with a single maximum at $x=0.5$).
For each AF, we first trained a preferential GP model using 5 randomly chosen comparisons (left column of Figure~\ref{fig:qeubo_vs_qei}), then generated 5 more (middle column of Figure~\ref{fig:qeubo_vs_qei}), and additionally 5 more comparisons (right column of Figure~\ref{fig:qeubo_vs_qei}) using qEUBO (top rows) and qEI (bottom rows) respectively.
After the first 5 randomly generated comparisons, the posteriors are the same, the contours of the two AFs are similar because $I_n$ is small, and the two methods make similar queries.
After 5 more queries, generated using each AF, qEUBO has already learned that $0.5$ is a good solution and is comparing this with other alternatives. In contrast, qEI is choosing not to compare with $0.5$. This pattern continues after 5 additional queries.


\subsection{Convergence Analysis of qEUBO and qEI}
We end this section by discussing the convergence properties of qEUBO and qEI. We show that, under sufficient regularity conditions, qEUBO's Bayesian simple regret converges to zero at a rate $o(1/n)$. We also show that there are problem instances where qEI has Bayesian simple regret bounded below by a positive constant; in particular, qEI is not asymptotically consistent.

Our analysis assumes that $\X$ is finite, $q=2$, and other technical conditions described in Section~\ref{sec:thm3_4}. These conditions hold in a broad range of settings. For example, they hold under the logistic likelihood function discussed above if the prior distribution on $f$ is such that 
\begin{align*}
    \delta \leq |f(x) - f(y)| \leq \Delta
\end{align*}
almost surely whenever $x\neq y$ for some $\Delta\geq\delta>0$. They also hold for general non-degenerate GP prior distributions if the likelihood function satisfies
\begin{equation*}
L(f(X);\lambda) = a
\end{equation*}
for some fixed $a > 1/2$ whenever $f(x_1) \neq f(x_2)$.

Under these conditions, we show the following results. Formal statements and proofs can be found in Section~\ref{sec:thm3_4}.

\begin{theorem}
\label{thm:eubo_conv}
Assume  the sequence of queries is chosen by maximizing qEUBO  and the assumptions described in Section~\ref{sec:thm3_4} hold. Then, $\E[f(x^*) - f(\widehat{x}_n^*)] = o(1/n)$, where $x^* = \argmax_{x\in\X}f(x)$ and $\widehat{x}_n^* \in \argmax_{x\in\X} \E_n[f(x)]$.
\end{theorem}

\begin{theorem}
\label{thm:qEI}
There exists a problem instance (i.e., $\X$ and Bayesian prior distribution over $f$) satisfying the assumptions described in Section~\ref{sec:thm3_4} such that if the sequence of queries is chosen by maximizing  qEI, then $\E[f(x^*) - f(\widehat{x}_n^*)] \geq R$ for all $n$, for a constant $R > 0$.
\end{theorem}

The problem instance in which qEI fails to be consistent has the characteristics previously described in Section~\ref{sec:qEI} --- the incumbent, i.e., the alternative with the best posterior mean, also has known value. As a result, qEI is unwilling to include it in the queries asked.  This makes qEI unable to learn about the value of other alternatives --- it can learn about the relative value of other alternatives with {\it each other}, but not about their value relative to the incumbent. 

\section{EXPERIMENTS}
\label{sec:experiments}
We compare qEUBO with various state-of-the-art AFs for PBO from the literature. We consider MPES from \cite{nguyen2021top}, which, as described, is arguably the only existing PBO AF with a proper justification. We also consider qEI and batch Thompson sampling (qTS) from \cite{siivola2021preferential}, which were both shown to have excellent empirical performance.   We also consider qNEI, a version of qEI that accounts for the uncertainty in latent function values through Monte Carlo integration over fantasized values \cite{balandat2020botorch}. qEUBO, qEI, and qNEI are optimized via sample average approximation with multiple restarts \citep{balandat2020botorch}. qTS uses approximate sample paths obtained via 1000 random Fourier features \citep{rahimi2007random}. For reference, we also include the performance of random search (Random), which selects queries uniformly at random over the space of alternatives. All algorithms use a Gaussian process prior with a constant mean function and RBF covariance function to model $f$. We approximate the posterior distribution over $f$ via the variational inducing point approach introduced by \cite{hensman2015scalable}. Our approach is equivalent to the one pursued by \cite{nguyen2021top} if we take the set of inducing points equal to the set of all points in the queries asked so far. Our set of inducing points includes these points in addition to a small set of quasi-random Sobol points, which improves performance slightly. 

\begin{table*}[th]
\centering
\begin{tabular}{l|rrrrr}
\toprule
Problem/Acquisition function & qNEI & MPES & qTS & qEI & qEUBO\\
\midrule
Ackley &    8.1 &  24.8 & 6.5 &   6.3 &  12.4  \\
Alpine1 &   11.4 & 16.4 &  6.4 & 8.9 &   11.3  \\
Hartmann &  8.7 & 15.3 & 7.1  &  6.0 &    8.3  \\
Animation &  9.4 & 13.5 & 8.6  &  7.4 &    8.2 \\
Carcab  & 7.2  & 12.7 & 6.9   &  7.1   & 7.2\\
Sushi &  8.9 & 23.9 & 9.5  &  5.9 &    7.5  \\
\bottomrule
\end{tabular}
    \caption{Average runtimes in seconds across all test problems. MPES is consistently the slowest algorithm, followed by qNEI. MPES is slow because it requires approximating an intractable integral involving the posterior distribution on the utility function's maximizer.  qTS and qEI are the fastest algorithms, followed closely by qEUBO.
    }
    \label{tab:runtimes}
\end{table*}


We report results across three synthetic test functions and three test functions built from real-world data. In contrast with most existing papers from the literature, which limit themselves to low-dimensional problems, we focus on more challenging problems of moderate dimension ($>3$). Synthetic functions include 6-dimensional Ackley , 7-dimensional Alpine1, and 6-dimensional Hartmann. Realistic problems include a 7-dimensional car cab design problem (Carcab) \citep{lin2022bope}, a 4-dimensional problem involving real-world human preferences over 100 sushi items (Sushi) \citep{siivola2021preferential},  and a novel 5-dimensional animation optimization problem (Animation). Noise is added to simulate inconsistency in the DM's responses.

To create our novel animation optimization problem, we use real human comparison data from a real-world particle effect rendering animation based on the publicly available demo in the AEPsych package~\citep{owen2021aepsych}.
In this setting, a human user is asked to compare two rendered animations of particles side by side and to determine which one looks more like fire (Figure~\ref{fig:particle_demo}, top). The particle animation is parameterized by 5 parameters.
We collected 100 such pairwise comparisons from human users with random particle animation parameters. We then confirmed that by fitting a support vector machine model on this data
and optimizing, we are able to obtain a realistic fire-like particle effect. A screenshot of the resulting animation shown in the bottom of Figure~\ref{fig:particle_demo}.
We then use this fitted model as the ground-truth test function to perform simulation.


\begin{figure}[ht]
\centering
\begin{tabular}[b]{c}%
\includegraphics[width=0.42\textwidth]{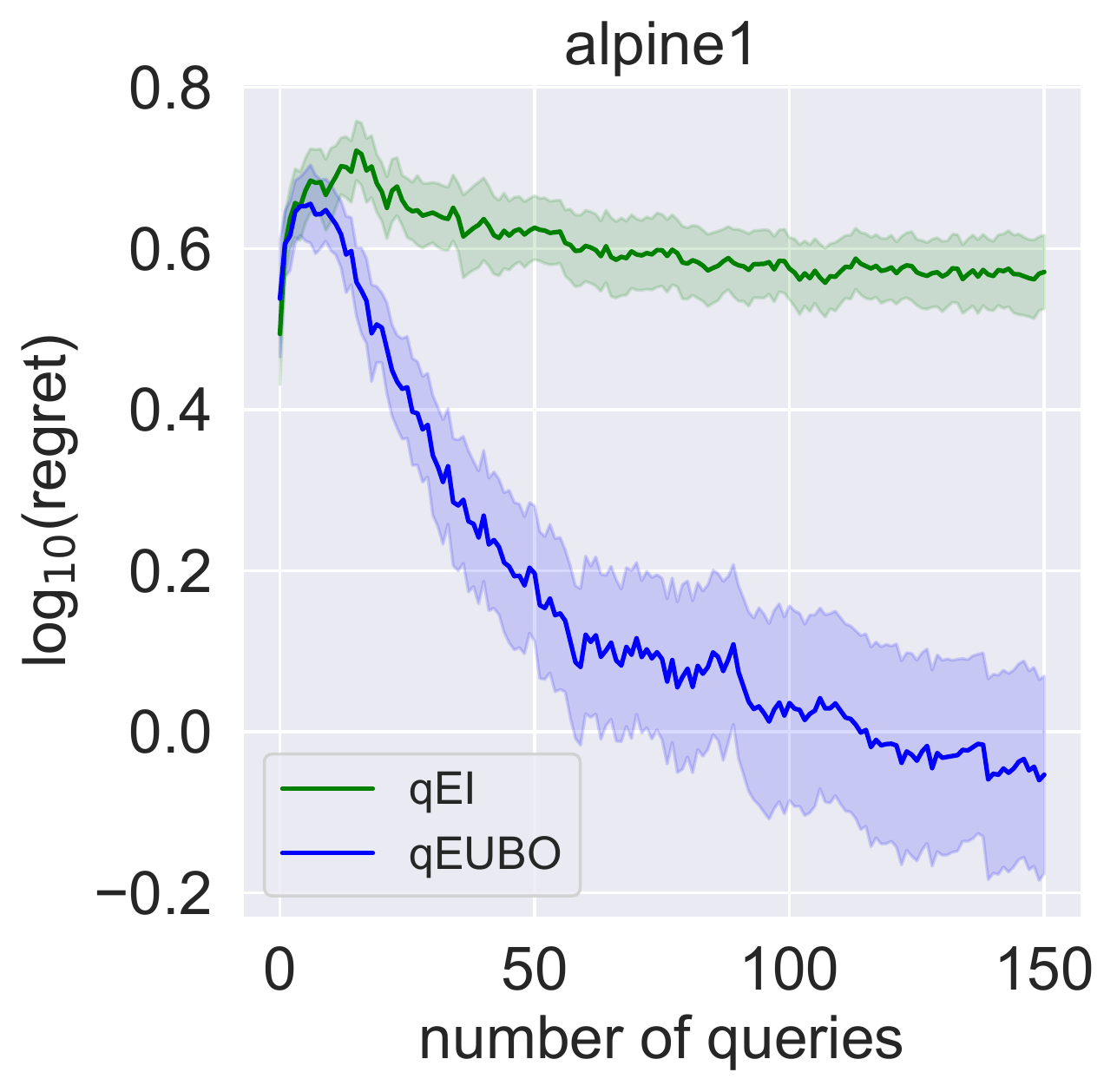}
 \end{tabular}
 \caption{Comparison between qEUBO and qEI on the 7-dimensional Alpine1 function seeded with many comparisons between a good solution and other randomly chosen ones. This setting is similar to Theorem~\ref{thm:qEI}. When we have a reasonably good status quo solution whose value is known with high precision, qEI is unable to significantly reduce its simple regret while qEUBO steadily learns.
 \label{fig:baseline}}
\end{figure}

In all problems, a first stage of interaction with the DM is performed using $4d$ queries chosen uniformly at random over $\X^q$, where $d$ is the input dimension of the problem. After this initial stage, each algorithm was used to select 150 additional queries sequentially. Figures~\ref{fig:lr_results} and~\ref{fig:lr_batch_results} show the mean of the log simple regret, plus and minus 1.96 times the standard deviation divided by the square root of
the number of replications, as a function of the number of queries. Here, simple regret is is defined as the maximum objective value minus the objective value at the maximizer of the posterior mean. We average over 100 replications for the Animation and Sushi problems and 50 replications for the other problems. 
Figure~\ref{fig:lr_results} shows results for $q=2$ for all algorithms. Figure~\ref{fig:lr_batch_results} shows results for both $q=2$ and $q=4$ for MPES, qTS, qEI, and qEUBO; we only focus on the best-performing algorithms to reduce visual clutter. All problems use moderate levels of Gumbel noise, consistent with the use of a logistic likelihood.  qEUBO outperforms all other AFs in all problems except Carcab for $q=2$, followed by qEI and then by qTS. In Section~\ref{sec:addit} we also include results for $q=4$, $q=6$, and varying levels of noise. In these results, qEUBO continues to consistently outperform competitor methods. 

Figure~\ref{fig:lr_batch_results} shows that including more alternatives in each query ($q=4$ vs. $q=2$) allows qEUBO to achieve a given simple regret using fewer queries.
Other AFs also benefit from including more alternatives in each query, but qEUBO seems to benefit the most.
This contrasts with \citet{siivola2021preferential}, which found only a marginal benefit of using larger values of $q$. At the same time, our results are consistent with those from \cite{mikkola2020projective}, which also observed significant benefits from using queries with \textit{larger information content}.  Our work provides complementary evidence because each query in \cite{mikkola2020projective} is equivalent to an infinite number of pairwise comparisons, while our queries use only $q-1$ comparisons. Results in Section~\ref{sec:addit} suggest that there is also a benefit in going from $q=4$ to $q=6$ for all algorithms considered there, including qEUBO, but that this benefit is smaller and less consistent than that of going from $q=2$ to $q=4$.

Table~\ref{tab:runtimes} shows the AF optimization walltime per iteration for each AF and each test problem, averaged over all the iterations. 
qEI is competitive in terms of its computational requirements, often  outperforming all the other AFs.
qEUBO is fast enough to support interactive learning applications, such as those for psychophysics experimentation \citep{owen2021aepsych} and animation \citep{brochu2010bayesian}, despite the challenging dimensionality of the experiments presented here. To better support interactive applications, one can begin optimizing qEUBO to generate the next query while the user is considering the current query. This can be done by initiating qEUBO for all possible user responses to the current query.


Figure~\ref{fig:baseline} compares qEI and qEUBO on an example problem similar to the one analyzed in Theorem~\ref{thm:qEI} in which qEI fails to be consistent. 
The objective function is the 7-dimensional Alpine1 function.
The initial data set contains several queries constituted by pairs where the first point is a known high-utility point close to the optimum, and the second point is drawn uniformly at random over the domain. 
After these comparisons, the value of this point has relatively low variance and has a posterior mean relatively high compared to the posterior mean elsewhere.
This mimics a setting common in practice where we have an in-use status quo solution that is reasonably good, has well-understood performance because it is currently in use, and on which we would like to improve. 
In this setting, the incumbent value $I_n$ has a reasonably high value relative to the posterior mean elsewhere and the variance of the latent utility near the incumbent solution is small. As a result, qEI does not include the incumbent solution or nearby values in DM queries, hampering its ability to learn. 
Consequently, qEI's simple regret stalls while qEUBO, on the other hand, makes steady progress as the number of queries grows. 

The code used to conduct our empirical evaluation can be found at \url{https://github.com/facebookresearch/qEUBO}. 

\section{CONCLUSION}
This work introduces the expected utility of the best option (qEUBO) acquisition function for preferential Bayesian optimization. qEUBO is simple to compute, has a sound decision-theoretic interpretation, and exhibits a strong empirical performance across a broad range of problems. We also draw a connection between qEUBO and its closest competitor, qEI, showing that qEI tends to perform well in early iterations because it is similar to qEUBO but its performance degrades as the number of queries grows or when the variance around the optimum is very small. Furthermore, we show that qEUBO's Bayesian simple regret converges to zero at a rate $o(1/n)$ as the number of queries, $n$, goes to infinity.  In contrast, we show that simple regret under qEI can fail to converge to zero. Finally, we demonstrate the substantial benefit of performing queries with more than two alternatives, in contrast with previous work, which found only a marginal benefit. Future directions include studying qEUBO's performance under other probabilistic models and extending qEUBO to more structured problem settings such as contextual preferential Bayesian optimization.

\subsubsection*{Acknowledgements}
We thank Stryker Buffington and Michael Shvartsman for their help setting up the animation example. We also thank the anonymous
reviewers for their helpful comments.
PF was supported by 
AFOSR FA9550-19-1-0283 and FA9550-20-1-0351-20-1-0351.

\bibliography{ref}

\renewtheorem{theorem}{Theorem}[section]
\newtheorem{lemma}{Lemma}[section]
\appendix
\onecolumn

\section{PROOFS OF THEOREMS 1 AND 2}
\label{sec:thm1_2}
\subsection{Proof of Theorem 1}
\begin{theorem}[Theorem 1]
Suppose the DM's responses are noise-free. Then, $\argmax_{X\in\X^q}\eubo_n(X) \subseteq \argmax_{X\in\X^q}V_n(X)$.
\end{theorem}
\begin{proof}
 For any given $X\in\X^q$ and each $i\in\{1,\ldots, q\}$ let $x^+(X, i) \in \argmax_{x\in \X} 
\E_n\left[f(x) \mid (X, i)\right]$ and define $X^+(X) = (x^+(X, 1), \ldots, x^+(X, q))$. We claim that
\begin{equation}
\label{eq:lemma}
V_n(X) \leq \eubo_n(X^+(X)).
\end{equation}
To see this, note that 
\begin{align*}
V_n(X)
= & \sum_{i=1}^q \prob_n(r(X)=i) \E_n[f(x^+(X, i)) | (X, i)]\\
\le {} &
\sum_{i=1}^q \prob_n(r(X)=i) \E_n\left[\max_{i=1,\ldots, q}f(x^+(X, i)) | (X, i)\right] \\
= {} & \E_n[\max_{i=1,\ldots, q}f(x^+(X, i))\}]\\
= {} & \eubo_n(X^+(X)),
\end{align*}
as claimed.

On the other hand, for any given $X\in\X^q$ we have
\begin{align*}
    \E_n[f(x_{r(X)})\mid (X, r(X))] \leq \max_{x\in\X}\E_n[f(x)\mid (X, r(X))].
\end{align*}
Since $f(x_{r(X)}) = \max_{i=1,\ldots, q}f(x_i)$,  taking expectations over $r(X)$ on both sides we obtain
\begin{equation}
\begin{split}
\eubo_n(X)\leq V_n(X).
\end{split}
\label{eq:lemma2}
\end{equation}

Now, building on the arguments above, let $X^*\in\argmax_{X\in\X^q}\eubo_n(X)$ and suppose for the sake of contradiction that $X^* \notin \argmax_{X\in\X^q}V_n(X)$. Then, there exists $\widetilde{X}\in\X^q$ such that $V_n(\widetilde{X}) > V_n(X^*)$. By the arguments above we have
\begin{align*}
\eubo_n(X^+(\widetilde{X})) &\geq 
V_n(\widetilde{X})\\
& >  V_n(X^*)\\
& \ge \eubo_n(X^*)\\ 
& \ge \eubo_n(X^+(\widetilde{X})).
\end{align*}
The first inequality follows from \eqref{eq:lemma}. 
The second inequality is due to our supposition for contradiction.
The third inequality is due to \eqref{eq:lemma2}.
Finally, the fourth inequality holds since  $X^*\in\argmax_{X\in\X^q}\eubo_n(X)$.

This contradiction concludes the proof.
\end{proof}

\subsection{Proof of Theorem 2}
Before proving Theorem\ref{thm:2}, we introduce notation and prove several lemmas. 

Throughout this section we assume that
\begin{equation*}
    \prob(r(X)=i\mid f(X)) = \frac{\exp(f(x_i)/\lambda)}{\sum_{j=1}^q \exp(f(x_j)/\lambda)},
\end{equation*}
for $i = 1,\ldots, q$. We also define the functions
\begin{equation*}
    U_n^\lambda(X) = \E_n[f(x_{r(X)})],
\end{equation*}
and
\begin{equation*}
    V_n^\lambda(X) = \E_n\left[\max_{x\in\X}\E_n[f(x) \mid (X, r(X))]\right].
\end{equation*}
 We note that $V_n^\lambda$ makes the dependence of $V_n$ on $\lambda$ explicit. On the other hand, $U_n^\lambda$ generalizes the definition of $\eubo_n$, which is obtained as a special case for $\lambda=0$.

Our analysis is similar to the one pursued by \cite{viappiani2010optimal} to relate optimal recommendation sets and optimal query sets. In particular, we leverage the following result, which can be deduced from the proof of Theorem 3 in the supplement of \cite{viappiani2010optimal}.
\begin{lemma}
\label{lemma:a1}
For any $s_1,\ldots, s_q \in \R$,
\begin{align*}
    \sum_{i=1}^q \frac{\exp(s_i/\lambda)}{\sum_{j=1}^q\exp(s_j/\lambda)}s_i \geq \max_{i=1,\ldots, q} s_i - \lambda C,
\end{align*}
where $C = L_W((q-1)/e)$, where $L_W$ is the Lambert $W$ function \citep{corless1996lambertw}.
\end{lemma}
\begin{proof}
 We may assume without loss of generality that $\max_{i=1,\ldots, q}s_i = s_q$. Let $t_i = (s_q - s_i)/\lambda$ for $i=1,\ldots, q-1$. After some algebra, we see that the inequality we want to show is equivalent to
\begin{align*}
\sum_{i=1}^{q-1}\frac{t_i\exp(-t_i)}{1 + \sum_{j=1}^{q-1}\exp(-t_j)}\leq  C.
\end{align*}
Thus, it suffices to show that the function $\eta:[0,\infty)^{q-1}\rightarrow\R$ given by 
\begin{align*}
\eta(t_1,\ldots, t_{q-1}) = \sum_{i=1}^{q-1}\frac{t_i\exp(-t_i)}{1 + \sum_{j=1}^{q-1}\exp(-t_j)}
\end{align*}
is bounded above by $C$. We refer the reader to the supplement of \cite{viappiani2010optimal} for a proof.
\end{proof}

\begin{lemma}
\label{lemma:a2}
$U_n^\lambda(X) \geq \eubo_n(X) - \lambda C$ for all $X\in\X^q$.
\end{lemma}
\begin{proof}
Note that
\begin{align*}
 \E[f(x_{r(X)})\mid f(X)] = \sum_{i=1}^q\frac{\exp(f(x_i)/\lambda)}{\sum_{j=1}^q\exp(f(x_j)/\lambda)}f(x_i).  
\end{align*}
Thus, Lemma~\ref{lemma:a1} implies that
\begin{equation*}
\E[f(x_{r(X)})\mid f(X)] \geq \max_{i=1,\ldots, q}f(x_i) - \lambda C.    
\end{equation*}
Taking expectations over both sides of the inequality yields the desired result.
\end{proof}

\begin{lemma}
\label{lemma:a3}
$V_n^\lambda(X) \geq U_n^\lambda(X)$ for all $X\in\X^q$.
\end{lemma}
\begin{proof}
Observe that
\begin{align*}
    V_n^\lambda(X) &= \E_n\left[\max_{x\in \X}\E[f(x) \mid (X, r(X))]\right]\\
    & \geq \E_n\left[\E[f(x_{r(X)}) \mid (X, r(X))]\right]\\
    &= \E_n[f(x_{r(X)})]\\
    &= U_n^\lambda(X),
\end{align*}
where the penultimate equality follows by the law of iterated expectation.
\end{proof}

\begin{theorem}[Theorem 2]
If $X^* \in \argmax_{X \in \X^q} \eubo_n(X)$, then $V_n^\lambda(X^*) \geq \max_{X\in\X^q}V_n^0(X) - \lambda C$.
\end{theorem}
\begin{proof}
Let $X^{**} \in \argmax_{X\in \X^q} V_n^0(X)$. We have the following chain of inequalities:
\begin{align*}
V_n^\lambda(X^*) &\geq  U_n^\lambda(X^*)\\
&\geq U_n^0(X^*) - \lambda C\\
& = \eubo_n(X^*) - \lambda C\\
& \geq \eubo_n(X^+(X^{**})) - \lambda C\\
& \geq V_n^0(X^{**}) - \lambda C\\
& = \max_{X\in \X^q} V_n^0(X) - \lambda C.
\end{align*}
The first inequality follows from Lemma~\ref{lemma:a3}. The second inequality follows from Lemma~\ref{lemma:a2}. The third line (first equality) follows from the definition of $U_n^0$. The fourth line (third inequality) follows from the definition of  $X^*$. The fifth line (fourth inequality) can be obtained as in the proof of Theorem 1. Finally, the last line (second equality) follows from the definition of $X^{**}$.
\end{proof}
\section{PROOFS OF THEOREMS 3 AND 4}
\label{sec:thm3_4}
In this section, we prove Theorems 3 and 4. In contrast with other sections, here we use super indices with parentheses to denote the iteration number, $n$. We denote the query presented to the used by $X^{(n)}$, and the corresponding user's response by $r(X^{(n)})$. We shall sometimes denote this response more compactly by $r^{(n)}$. We let $\D^{(n)} = \{(X^{(m)}, r^{(m)})\}_{m=1}^n$ denote the data collected  up to time $n$. Similarly, we denote the conditional expectation given $\D^{(n)}$ by $\E^{(n)}$.

Recall that Theorems~\ref{thm:eubo_conv} and ~\ref{thm:qEI} assume that $\X$ is finite and and $q=2$. We make these assumptions throughout this section without stating them explicitly. We also assume that $\E[\max_{x\in\X}|f(x)|] < \infty$, which guarantees that all expectations involved in our analysis are finite.

The statement of Theorems 3 and 4 rely on the following assumptions (called Assumptions 1-3 respectively):
\begin{enumerate}
    \item $\prob(f(x) = f(y))=0$ for any $x,y\in\X$ with $x\neq y$.
    \item There exists $a > 1/2$ such that $\prob(r(X) \in \argmax_{i=1,\ldots, 2}f(x_i) \mid f(X)) \geq a$ for any $X=(x_1, x_2) \in \X^2$ with $x_1\neq x_2$ almost surely under the prior on $f$.
    \item There exist $\Delta \geq \delta > 0$ such that for any $\D^{(n)}$ and any $x, y \in \X$ (potentially depending on $\D^{(n)}$),
\begin{equation*}
\delta \prob^{(n)}(f(x) > f(y)) \leq \E^{(n)}[\{f(x) - f(y)\}^+]   \leq  \Delta \prob^{(n)}(f(x) > f(y))
\end{equation*}
almost surely under the prior on $f$.
\end{enumerate}

Before we begin with the proofs of Theorems 3 and 4, we show that  Assumptions 1-3 hold in two general settings. This is summarized in the following two lemmas. 


First, we show that if the difference in utilities between items is bounded below (away from 0) and above almost surely under the prior, then Assumptions 1-3 hold for a broad class of likelihoods for the query response.
\begin{lemma}
\label{lemma:b1}
Suppose that there exist $\Delta \geq \delta > 0$ such that $\delta \leq |f(x) - f(y)| \leq \Delta$ almost surely whenever $x\neq y$.
Additionally, suppose that the likelihood function 
$L_i(f(X);\lambda) = \prob(r(X)=i\mid f(X))$ is 
bounded below by $a>1/2$ when $\delta \leq |f(x) - f(y)|$. 
Then, Assumptions 1-3 hold.
\end{lemma}

\begin{proof} Since $\delta \leq f(x) - f(y)|$  almost surely whenever $x\neq y$, Assumption 1 holds trivially. The assumed lower bound on the likelihood implies Assumption~2.

Now we show that Assumption 3 holds. Observe that
\begin{align*}
\{f(x) - f(y)\}^+ &= (f(x) - f(y))1\{f(x) > f(y)\}\\
&\geq \delta 1\{f(x) > f(y)\}    
\end{align*}
almost surely under the prior (and thus under the posterior at time $n$) on $f$. Thus,
\begin{align*}
\E^{(n)}\left[\{f(x) - f(y)\}^+\right] &\geq \E^{(n)}\left[\delta 1(f(x) > f(y)\right]\\
&= \delta \prob^{(n)}(f(x) > f(y)).
\end{align*}

We can show similarly that
\begin{equation*}
\E^{(n)}\left[\{f(x) - f(y)\}^+\right] < \Delta \prob^{(n)}(f(x) > f(y)).    
\end{equation*}
\end{proof}

This lemma's lower bound on the likelihood is satisfied for the logistic likelihood, the probit likelihood, or for any other likelihood in which $\prob(r(X)=1 \mid f(X))$ is a strictly increasing function of $f(x_1)-f(x_2)$ and equal to $1/2$ when $f(x_1) = f(x_2)$.

Our next result considers the situation where the difference in item utilities is not bounded.
In this situation, it shows that Assumptions 1-3 hold 
provided the DM's response likelihood satisfies a stronger condition than that assumed by Lemma~\ref{lemma:b1}: the DM's responses are correct with constant probability greater than $1/2$ whenever the items in the query have non-identical utility value. 
This result applies to general non-degenerate GP prior distributions; i.e., for GP prior distributions with a positive definite covariance function.
\begin{lemma}
\label{lemma:b2}
Suppose that Assumption 1 holds and there exists $a > 1/2$ such that $\prob(r(X) = \argmax_{i=1,\ldots, 2}f(x_i) \mid f(X)) = a$ whenever $f(x_1)\neq f(x_2)$. Then, Assumption 3 holds.
\end{lemma}

\begin{proof}
We first make a fundamental observation. For a permutation $\pi: \{1, \ldots, |\X|\} \rightarrow \{1, \ldots, |\X|\}$, let $A_\pi$ denote the event $\{f(x_{\pi(1)}) < \cdots < f(x_{\pi(|\X|)})\}$. We note that the distribution of $f$ given $\D^{(n)}$ and $A_\pi$ coincides with the distribution of $f$ given $A_\pi$. This observation may be surprising at first glance, but it has a very intuitive interpretation. Since the DM's responses depend exclusively on the relative order of the utility values, they only provide ordinal information about the utility values. Thus, if the relative order of the utilities of all points is known, we learn nothing from the DM's responses; i.e., the conditional distribution given $A_\pi$ remains unchanged if we observe $\D^{(n)}$.  This property relies heavily on the choice of the likelihood and does not hold in general for other likelihoods such as the logistic likelihood. We will make use of this observation in the proof of this lemma.

We first prove the existence of $\delta$. Let $\Pi$ be the set of all permutations over $\{1,\ldots, |\X|\}$ and let $\pi$ be an arbitrary element of $\Pi$. Define
\begin{equation*}
\gamma(x,y, \pi) = \E[\{f(x) - f(y)\}^+ | A_\pi].  
\end{equation*}
We note that $\E[\{f(x) - f(y)\}^+ | A_\pi] > 0$ if and only if $f(x) > f(y)$ under $A_\pi$. 

Let
\begin{equation*}
 \delta = \min_{x,y\in\X, \pi\in\Pi}\{\gamma(x,y, \pi) : \gamma(x,y, \pi) > 0\},
\end{equation*}
We have
\begin{align*}
\E^{(n)}[\{f(x) - f(y)\}^+] &= \sum_{\pi\in\Pi}\prob^{(n)}(A_\pi)\E^{(n)}[\{f(x) - f(y)\}^+ \mid A_\pi]\\
&= \sum_{\pi\in\Pi}\prob^{(n)}(A_\pi)\E[\{f(x) - f(y)\}^+ \mid A_\pi]\\
&\geq \sum_{\pi\in\Pi}\prob^{(n)}(A_\pi) 1\{f(x) > f(y) \mid A_\pi\} \delta\\
&= \delta\prob^{(n)}(f(x) > f(y)),
\end{align*}
where the second equation follows from the observation earlier that $f$ and the DM's responses are conditionally independent given $A_\pi$ and the inequality follows from the observation above that $\E[\{f(x) - f(y)\}^+ | A_\pi] > 0$ if and only if $f(x) > f(y)$ under $A_\pi$.

The existence of $\Delta$ can be proved similarly by taking
\begin{equation*}
\Delta = \max_{x,y\in\X, \pi\in\Pi}\{\delta(x,y, \pi) : \delta(x,y, \pi) > 0\}.    
\end{equation*}
\end{proof}

\subsection{Proof of Theorem 3}
The proof of Theorem~\ref{thm:eubo_conv} is achieved via a series of lemmas. We assume that $$X^{(n+1)} = (x_1^{(n)}, x_2^{(n)}) \in\argmax_{X\in \X^2}\eubo^{(n)}(X)$$ for all $n$ throughout the proofs of these lemmas. We  define $x^* = \argmax_{x\in\X}f(x)$, and $x_*^{(n)} = \argmax_{x\in\X}\E^{(n)}[f(x)]$. 

\begin{lemma}
\label{lemma:b3}
Let $p_{**}^{(n)} = \prob^{(n)}(x_*^{(n)}\neq x_*)$. Then, 
\begin{equation*}
    \max_{X\in \X^2}\eubo^{(n)}(X)  \geq \max_{x \in \X}\E^{(n)}[f(x)] + \frac{\delta}{|\X|-1}p_{**}^{(n)}.
\end{equation*}
\end{lemma}
\begin{proof}
Observe that 
\begin{align*}
\max_{X\in \X^2}\eubo^{(n)}(X) &\geq \max_{x\in\X}\eubo^{(n)}(x, x_*^{(n)})\\
&= \max_{x\in\X}\E^{(n)}[\max\{f(x), f(x_*^{(n)})\}]\\
&= \max_{x\in\X}\E^{(n)}[f(x_*^{(n)}) + \{f(x) - f(x_*^{(n)})\}^+]\\
&= \E^{(n)}[f(x_*^{(n)})] + \max_{x\in\X}\E^{(n)}[\{f(x) - f(x_*^{(n)})\}^+]\\
&\geq \max_{x \in \X}\E^{(n)}[f(x)] + \max_{x\in\X} \delta\prob^{(n)}(f(x) > f(x_*^{(n)}))\\
&\geq \max_{x \in \X}\E^{(n)}[f(x)] + \frac{\delta}{|\X|-1}\sum_{x\in\X}\prob^{(n)}(f(x) > f(x_*^{(n)}))\\
&\geq \max_{x \in \X}\E^{(n)}[f(x)] + \frac{\delta}{|\X|-1}\prob^{(n)}(\cup_{x\in\X}\{f(x) > f(x_*^{(n)})\})\\
&= \max_{x \in \X}\E^{(n)}[f(x)] + \frac{\delta}{|\X|-1}\prob^{(n)}(x_*^{(n)}\neq x_*),
\end{align*}
where the second inequality follows from Assumption 2 and the last inequality follows from the union bound.
\end{proof}

\begin{lemma}
\label{lemma:b4}
Let $p_1^{(n)} = \prob^{(n)}(f(x_1^{(n+1)}) > f(x_2^{(n+1)}))$ and $p_2^{(n)} = \prob^{(n)}(f(x_2^{(n+1)}) > f(x_1^{(n+1)}))$. Then,
\begin{equation*}
p_1^{(n)}, \ p_2^{(n)} \geq C p_{**}^{(n)},   
\end{equation*}
where $C= \frac{\delta}{(|\X|-1)\Delta}$.
\end{lemma}
\begin{proof}
We have
\begin{align*}
\eubo^{(n)}(X^{(n+1)}) &=  \E^{(n)}[f(x_1^{(n+1)})] + \E^{(n)}[\{f(x^{(n+1)}_2) - f(x^{(n+1)}_1)\}^+] \\
&\leq \E^{(n)}[f(x_1^{(n+1)})] +  \Delta\prob^{(n)}(f(x_2^{(n+1)}) > f(x_1^{(n+1)}))\\
&\leq \max_{x \in \X}\E^{(n)}[f(x)] + \Delta\prob^{(n)}(f(x_2^{(n+1)}) > f(x_1^{(n+1)})),
\end{align*}
where the first inequality is again due to Assumption 2. Combining this with Lemma~\ref{lemma:b2}., we get
\begin{equation*}
    \max_{x \in \X}\E^{(n)}[f(x)] + \Delta\prob^{(n)}(f(x_2^{(n+1)}) > f(x_1^{(n+1)})) \geq \max_{x \in \X}\E^{(n)}[f(x)] + \frac{\delta}{|\X|-1}p_{**}^{(n)};
\end{equation*}
i.e.,
\begin{equation*}
    \prob^{(n)}(f(x_2^{(n+1)}) >  f(x_1^{(n+1)}) \geq \frac{\delta}{(|\X|-1)\Delta}p_{**}^{(n)}.
\end{equation*}
Finally, it follows by symmetry that
\begin{equation*}
    \prob^{(n)}(f(x_1^{(n+1)}) >  f(x_2^{(n+1)}) \geq \frac{\delta}{(|\X|-1)\Delta}p_{**}^{(n)},
\end{equation*}
which finishes the proof.
\end{proof}


\begin{lemma}
\label{lemma:b5}
Let $X$ denote a generic discrete random vector and $Y$ denote a Bernoulli random variable correlated with $X$. Let $p_X$ and $p_Y$ denote the marginal distributions of $X$ and $Y$, respectively, and $p_{X|Y}$ denote the conditional distribution of $X$ given $Y$. 
Let $\entropy(X)$ and $\entropy(X|Y=y)$ denote the entropy of $X$ and the conditional entropy of $X$ given $Y=y$; i.e., 
\begin{equation*}
 \entropy(X) = -\sum_x p_X(x)\, \log(p_X(x)),   
\end{equation*}
and
\begin{equation*}
 \entropy(X|Y=y) = -\sum_x p_{X|Y}(x|y)\, \log(p_{X|Y}(x|y)).   
\end{equation*}

Define $q(x) = \prob(Y=0|X=x)$. Then,
\begin{equation*}
-\sum_{y=0}^1 p_Y(y) \entropy(X|Y=y) = \entropy(X) - \left[h\left(\sum_x p_X(x) q(x)\right) - \sum_x p_X(x) h(q(x))\right],
\end{equation*}
where $h(q) = -q\log q - (1-q) \log(1-q)$ is the binary entropy function.
\end{lemma}

\begin{proof}
Observe that $-\sum_y p_Y(y) \entropy(X|Y=y)$ is the conditional entropy of $X$ given $Y$ (see definition 2.10 in \cite{cover1999elements}). By basic information theory results (see equations 2.43 and 2.44 in \cite{cover1999elements}) we know that
\begin{equation*}
\entropy(X) - \entropy(X|Y) =  \entropy(Y) - \entropy(Y|X)    
\end{equation*}

Rearranging terms, we obtain
\begin{equation*}
 \entropy(X|Y) = \entropy(X) - (\entropy(Y) - \entropy(Y|X))   
\end{equation*}

Consider the term $\entropy(Y) - \entropy(Y|X)$. We have 
\begin{align*}
\prob(Y=0) &= \sum_x p_X(x)\prob(Y=0|X=x)\\
&= \sum_x p_X(x)q(x).
\end{align*}
Moreover, since $Y$ is a Bernoulli random variable, $\entropy(Y) = h(\prob(Y=0))$. 

The second term is
\begin{align*}
 \entropy(Y|X) &= \sum_x p_X(x) \entropy(Y|X=x)\\
 &= \sum_x p_X(x) h(q(x)).
\end{align*}

Hence,
\begin{equation*}
\entropy(Y) - \entropy(Y|X) = h(\prob(Y=0)) - \sum_x p_X(x) h(q(x)),    
\end{equation*}
which concludes the proof.

\end{proof}


\begin{lemma} 
\label{lemma:b6}
Enumerate the elements of $\X$ as $x_1,\ldots, x_{|\X|}$ and let $\pi_f$ be the random permutation satisfying $f(x_{\pi_f(1)}) < \cdots < f(x_{\pi_f(|\X|)})$. Let $p_{\pi_f}^{(n)}$ denote the posterior on $\pi$ given $\D^{(n)}$. Then,
\begin{equation*}
\E^{(n)}[\entropy(p_{\pi_f}^{(n+1)})]   \leq \entropy(p_{\pi_f}^{(n)}) -\varphi(p_1^{(n)})    
\end{equation*}
where the expectation on the left-hand side is over $r^{(n+1)}$, and $\varphi(u) = h(au + (1-a)(1-u)) - h(a)$.
\end{lemma}
\begin{proof}
Observe that $\E^{(n)}[\entropy(p^{(n+1)})] = \entropy(p^{(n)}\mid r^{(n+1)})$. Thus, from Lemma~\ref{lemma:b5} it follows that
\begin{equation*}
\E^{(n)}[\entropy(p^{(n+1)})] = \entropy(p^{(n)}) - \left[h\left(\sum_\pi p_{\pi_f}^{(n)}(\pi) q(\pi)\right) - \sum_\pi p_{\pi_f}^{(n)}(\pi) h(q(\pi))\right],   
\end{equation*}
where $q(\pi) = \prob^{(n)}(r^{(n+1)}=1 \mid \pi)$. Thus, it suffices to show that
\begin{equation*}
 h\left(\sum_\pi p_{\pi_f}^{(n)}(\pi) q(\pi)\right) - \sum_\pi p_{\pi_f}^{(n)}(\pi) h(q(\pi)) \geq \varphi(p_1^{(n)}).  
\end{equation*}

Let $\Pi_1$ and $\Pi_2$ be the set of permutations such that $f(x_1^{(n)}) < f(x_2^{(n)})$ and $f(x_1^{(n)}) > f(x_2^{(n)})$, respectively. Define
\begin{equation*}
q_i = \left(\sum_{\pi\in\Pi_i}  p_{\pi_f}^{(n)}(\pi) q(\pi)  \right)/p_i^{(n)}
\end{equation*}
for $i=1,2$. After some algebra we see that
\begin{align*}
h\left(\sum_\pi p_{\pi_f}^{(n)}(\pi) q(\pi)\right) - \sum_\pi p_{\pi_f}^{(n)}(\pi) h(q(\pi)) = h(p_1^{(n)} q_1 + p_2^{(n)} q_2) - p_1^{(n)}h(q_1) - p_2^{(n)}h(q_2) + p_1^{(n)}\psi_1 + p_2^{(n)}\psi_2,
\end{align*}
where
\begin{align*}
    \psi_i = h(q_i) \sum_{\pi\in\Pi_i}p_{\pi_f}^{(n)}(\pi)h(q(\pi))/p_1^{(n)}
\end{align*}
for $i=1,2$. Moreover, since $h$ is concave, $\psi_i \geq 0$ by Jensen's inequality. Thus,
\begin{align*}
h\left(\sum_\pi p_{\pi_f}^{(n)}(\pi) q(\pi)\right) - \sum_\pi p_{\pi_f}^{(n)}(\pi) h(q(\pi)) \geq h(p_1^{(n)} q_1 + p_2^{(n)} q_2) - p_1^{(n)}h(q_1) - p_2^{(n)}h(q_2) 
\end{align*}

Recall that $\prob(r(X)=\argmax_{i=1,2}f(x_i)\mid f(X))\geq a$ whenever $x_1 \neq x_2$ almost surely by Assumption 3. Also recall that $X^{(n+1)} = (x_1^{(n+1)}, x_2^{(n+1)}) \in \argmax_{X\in\X^2} \eubo_n(X)$. It is not hard to see that we can always choose $x_1^{(n+1)}$ and $x_2^{(n+1)}$ such that $x_1^{(n+1)}\neq x_2^{(n+1)}$. It follows from this that $q(\pi) \geq a$ for $\pi \in \Pi_1$ and $q(\pi) \leq 1 - a$ for $\pi \in \Pi_2$. From the definition of $q_i$ for $i=1,2$, this in turn implies that $q_1 \geq a$ and  $q_2 \leq 1 - a$.

Taking the derivative of  $h(p_1^{(n)} q_1 + p_2^{(n)} q_2) - p_1^{(n)}h(q_1) - p_2^{(n)}h(q_2)$ with respect to $q_1$ and recalling that the derivaitve of $h$ is decreasing since $h$ is concave, we can see that $h(p_1^{(n)} q_1 + p_2^{(n)} q_2) - p_1^{(n)}h(q_1) - p_2^{(n)}h(q_2)$ is minimal when $q_1 = a$ under the constraint $q_1 \geq a$. Similarly, we can see that $h(p_1^{(n)} q_1 + p_2^{(n)} q_2) - p_1^{(n)}h(q_1) - p_2^{(n)}h(q_2)$ is minimal when $q_2 = 1 - a$ under the constraint $q_2 \leq 1 - a$. Hence,
\begin{align*}
h\left(\sum_\pi p_{\pi_f}^{(n)}(\pi) q(\pi)\right) - \sum_\pi p_{\pi_f}^{(n)}(\pi) h(q(\pi)) &\geq    h(p_1^{(n)} a + p_2^{(n)} (1-a)) - p_1^{(n)}h(a) - p_2^{(n)}h(1-a) \\
&\varphi(p_1^{(n)}),
\end{align*}
where the last equation holds because $p_2^{(n)} = 1 - p_1^{(n)}$ and $h(a) = h(1-a)$. This concludes the proof.
\end{proof}

\begin{lemma}
\label{lemma:b7}
$\varphi(u) \geq 2(h(1/2) - h(a))u$ for $u\in[0,1/2]$.
\end{lemma}
\begin{proof}
Note that $\varphi$ is concave in $[0,1]$. Applying Jensen's inequality we obtain
\begin{align*}
    \varphi((1-2u)0 + (2u)(1/2) ) \geq (1-2u)\varphi(0) + 2u\varphi(1/2);
\end{align*}
i.e., 
\begin{align*}
    \varphi(u) \geq 2(h(1/2) - h(a))u.
\end{align*}
\end{proof}

\begin{lemma}
\label{lemma:b8}
Let $R^{(n)} = 2(h(1/2) - h(a))p_{**}^{(n)}C$. Then,
\begin{equation*}
\E^{(n)}[\entropy(p^{(n+1)})]  \leq \entropy(p^{(n)}) -  R^{(n)}.
\end{equation*}
\end{lemma}
\begin{proof}
From Lemma~\ref{lemma:b4} we know that $p_1^{(n)}, p_2^{(n)} \geq  p_{**}^{(n)}C$. Since $p_1^{(n)} + p_2^{(n)} = 1$, it follows that
\begin{equation*}
    0 \leq p_{**}^{(n)}C \leq \min\{p_1^{(n)}, p_2^{(n)}\} \leq 1/2.
\end{equation*}
Now observe that the function $\varphi$ is increasing in $[0,1/2]$ and symmetric around 1/2. Thus,
\begin{equation*}
 \varphi(p_{**}^{(n)}C) \leq  \varphi(\min\{p_1^{(n)}, p_2^{(n)}\}) = \varphi(p_1^{(n)}).
\end{equation*}

Finally,
\begin{align*}
    \E^{(n)}[\entropy(p^{(n+1)})]  &\leq   \entropy(p^{(n)}) -\varphi(p_1^{(n)})\\
    &\leq \entropy(p^{(n)}) -\varphi(p_{**}^{(n)}C)\\
    &\leq \entropy(p^{(n)}) -2(h(1/2) - h(a)) p_{**}^{(n)}C,
\end{align*}
where the first line comes from Lemma~\ref{lemma:b6}, the second line comes from the above analysis, and the third line is a consequence of Lemma~\ref{lemma:b7}.
\end{proof}

We are now in position to prove Theorem~\ref{thm:eubo_conv}.
\begin{theorem}[Theorem 3]
Suppose that Assumptions 1-3 are satisfied and $X^{(n+1)}\in \argmax_{X\in\X^q}\eubo^{(n)}(X)$ for all $n$. Then, $\E[f(x^*) - f(x_*^{(n)})] = o(1/n)$.
\end{theorem}
\begin{proof}
Consider the stochastic processes $\{Z^{(n)}\}_{n=0}^\infty$ defined by
\begin{equation*}
Z^{(n)} = \entropy(p^{(n)}) + \sum_{m=0}^{n-1}R^{(m)}, \ n\geq 0.  
\end{equation*}
Observe that $Z^{(n)}$ is non-negative for all $n$. Moreover,
\begin{align*}
    \E^{(n)}[Z^{(n+1)}] &=  \E^{(n)}[\entropy(p^{(n+1)})] + \sum_{m=0}^nR^{(n)}\\
    &\leq \entropy(p^{(n)}) - R^{(n)} + \sum_{m=0}^{n}R^{(m)}\\
    &= \entropy(p^{(n)}) + \sum_{m=0}^{n-1}R^{(m)}\\
    &= Z^{(n)},
\end{align*}
where the inequality above follows from Lemma~\ref{lemma:b8}. Thus, $\{Z^{(n)}\}_{n=0}^\infty$ is a non-negative supermartingale. By Doob's martingale convergence theorem, $\{Z^{(n)}\}_{n=0}^\infty$  converges almost surely to a random variable with finite expectation. This in turn implies that $\sum_{n=0}^\infty\E[ R^{(n)}] < \infty$. Since $R^{(n)} = 2(h(1/2) - h(a))p_{**}^{(n)}C$, it follows that $\sum_{n=0}^\infty\E[ p_{**}^{(n)}] <\infty$. Recall that $p_{**}^{(n)} =  \prob^{(n)}(x_*^{(n)}\neq x_*)$. By the law of the iterated expectation we obtain $\E[ p_{**}^{(n)}] = \prob(x_*^{(n)}\neq x_*)$. Hence, we have shown that $\sum_{n=0}^\infty \prob(x_*^{(n)}\neq x_*) < \infty$. We deduce from this that $\prob(x_*^{(n)}\neq x_*) = o(1/n)$.

Finally,
\begin{align*}
    \E[f(x_*) - f(x_*^{(n)})] &=  \E[\{f(x_*) - f(x_*^{(n)})\}^+]\\
    &\leq \Delta \prob(f(x_*) > f(x_*^{(n)}))\\
     &= \Delta \prob(x_*^{(n)}\neq x_*)\\
    &= o(1/n),
\end{align*}
where the first and third lines hold by definition of $x_*$, and the second line follows from Assumption 2.
\end{proof}


\subsection{Proof of Theorem 4}

\begin{theorem}[Theorem 4]
There exists a problem instance (i.e., $\X$ and Bayesian prior distribution over $f$) satisfying Assumptions 1-3 such that if $X^{(n+1)}\in \argmax_{X\in\X^q}\qei^{(n)}(X)$ for all $n$, then $\E[f(x^*) - f(\widehat{x}_*^{(n)})] \geq R$ for all $n$, for a constant $R > 0$.
\end{theorem}
\begin{proof}
Let $\X = \{1, 2, 3, 4\}$ and consider the functions $f_i:\X\rightarrow\R$, for $i=1,2,3,4$, given by $f_i(1) = -1$ and $f_i(2) = 0$ for all $i$, and
\begin{align*}
    f_1(x) = \begin{cases}
    1, &\ x=3\\
    \frac{1}{2}, &\ x=4
    \end{cases},
\hspace{0.5cm}
f_2(x) = \begin{cases}
    \frac{1}{2}, &\ x=3\\
    1, &\ x=4
    \end{cases},
\hspace{0.5cm}
f_3(x) = \begin{cases}
    -\frac{1}{2}, &\ x=3\\
    -1, &\ x=4
    \end{cases},
\hspace{0.5cm}
f_4(x) = \begin{cases}
    -1, &\ x=3\\
    -\frac{1}{2}, &\ x=4
    \end{cases}.
\end{align*}

Let $p$ be a number with $0 < p < 1/3$ and set $q=1-p$. We consider a prior distribution on $f$ with support $\{f_i\}_{i=1}^4$ such that
\begin{align*}
 p_i = \prob(f=f_i) =   \begin{cases}
                        p/2, \  i =1,2,\\
                        q/2, \ i=3,4.
                        \end{cases}
\end{align*}
We also assume the DM's response likelihood is given by $\prob(r(X)=1\mid f(x_1) > f(x_2)) = a$ for some $a$ such that $1/2 < a < 1$,

Let $\D^{(n)}$ denote the set of observations up to time $n$ and let $p_i^{(n)} = \prob(f=f_i \mid \D^{(n)})$ for $i=1,2,3,4$. We let the initial data set be $\D^{(0)} = \{(X^{(0)}, r^{(0)})\}$, where $X^{(0)}= (1,2)$. We will prove that the following statements are true for all $n\geq 0$.
\begin{enumerate}
    \item $p_i^{(n)} > 0$ for $i=1,2,3,4$.
    \item $p_1^{(n)} < \frac{1}{2}p_3^{(n)}$ and $p_2^{(n)} < \frac{1}{2}p_4^{(n)}$.
    \item $\argmax_{x\in\X}\E^{(n)}[f(x)]=\{2\}$.
    \item $\argmax_{X\in\X^2}\qei^{(n)}(X) = \{(3, 4)\}$.
\end{enumerate}
We prove this by induction over $n$. We begin by proving this for $n=0$. Since $f_i(1) < f_i(2)$ for all $i$, the posterior distribution on $f$ given $\D^{(0)}$ remains the same as the prior; i.e., $p_i^{(0)} = p_i$ for $i=1,2,3,4$. Using this, statements 1 and 2 can be easily verified. Now note that $\E^{(0)}[f(1)]=-1$, $\E^{(0)}[f(2)]=0$, and $\E^{(0)}[f(3)] = \E^{(0)}[f(4)] = \frac{3}{2}(p - q)$. Since $p < q$, it follows that $\argmax_{x\in\X}\E^{(n)}[f(x)]=\{2\}$; i.e., statement 3 holds. Finally, since $\max_{x\in\{1,2\}}\E^{(0)}[f(x)] = 0$, the qEI acquisition function at time $n=0$ is given by $\qei^{(0)}(X) = \E^{(0)}[\{\max\{f(x_1), f(x_2)\}\}^+]$. A direct calculation can now be performed to verify that statement 4 holds. This completes the base case.

Now suppose statements 1-4 hold for some $n\geq 0$. Since $X^{(n+1)} = (3, 4)$, the posterior distribution on $f$ given $\D^{(n+1)}$ is given by
\begin{align*}
p_i^{(n+1)} \propto \begin{cases}
                        p_i^{(n)}\ell, \ i=1,3,\\
                         p_i^{(n)} (1 - \ell), \ i=2,4,
                        \end{cases}
\end{align*}
where
\begin{equation*}
\ell = a \I\{r^{(n+1)} = 1\} + (1-a)\I\{r^{(n+1)} = 2\}.    
\end{equation*}
Observe that $0< \ell < 1$ since $0 < a < 1$. Thus, $\ell > 0$ and $1-\ell > 0$. Since $p_i^{(n)} > 0$ by the induction hypothesis, it follows from this that $p_i^{(n+1)} > 0$ for $i=1,2,3,4$. Moreover, since $p_i^{(n+1)} \propto p_i^{(n)}\ell$ for $i=1,3$ and $p_1^{(n)} < \frac{1}{2}p_3^{(n)}$ by the induction hypothesis, it follows that $p_1^{(n+1)} < \frac{1}{2}p_3^{(n+1)}$. Similarly, $p_2^{(n+1)} < \frac{1}{2}p_4^{(n+1)}$. Thus, statements 1 and 2 hold at time $n+1$.

Now observe that
\begin{align*}
    \E^{(n+1)}[f(3)] &= p_1^{(n+1)} + \frac{1}{2}p_2^{(n+1)} - \frac{1}{2}p_3^{(n+1)} - p_4^{(n+1)}\\
    &= \left(p_1^{(n+1)} - \frac{1}{2}p_3^{(n+1)}\right) + \left(\frac{1}{2}p_2^{(n+1)} - p_4^{(n+1)}\right)\\
    &\leq \left(p_1^{(n+1)} - \frac{1}{2}p_3^{(n+1)}\right) + \left(p_2^{(n+1)} - \frac{1}{2}p_4^{(n+1)}\right)\\
    &\leq 0,
\end{align*}
where the last inequality holds since $p_1^{(n+1)} < \frac{1}{2}p_3^{(n+1)}$ and $p_2^{(n+1)} < \frac{1}{2}p_4^{(n+1)}$. Similarly, we see that $\E^{(n+1)}[f(4)] \leq 0$. Since $\E^{(n+1)}[f(1)]=-1$ and $\E^{(n+1)}[f(2)]=0$, it follows that $\argmax_{x\in\X}\E^{(n+1)}[f(x)]=\{2\}$; i.e., statement 3 holds at time $n+1$.

Since $\max_{x\in\X}\E^{(0)}[f(x)] = 0$, the qEI acquisition function at time $n+1$ is given by $\qei^{(n+1)}(X) = \E^{(n+1)}[\{\max\{f(x_1), f(x_2)\}\}^+]$. Since $f(1) \leq f(x)$ almost surely under the prior for all $x\in\X$, there is always a maximizer of qEI that does not contain $1$.  Thus, to find the maximizer of qEI, it suffices to analyse its value at the pairs $(2, 3)$, $(3,4)$ and $(4,2)$. We have
\begin{equation*}
\qei^{(n+1)}(2, 3) = p_1^{(n+1)} + 1/2 p_2^{(n+1)},    
\end{equation*}
\begin{equation*}
\qei^{(n+1)}(3, 4) = p_1^{(n+1)} + p_2^{(n+1)}  
\end{equation*}
and
\begin{equation*}
\qei^{(n+1)}(4, 2) = 1/2p_1^{(n+1)} + p_2^{(n+1)}.   
\end{equation*}
Since $p_1^{(n+1)} > 0$ and $p_2^{(n+1)} > 0$, it follows that $\argmax_{X\in\X^2}\qei^{(n+1)}(X) = \{(3, 4)\}$, which concludes the proof by induction.

Finally, since $\argmax_{x\in\X}\E^{(n)}[f(x)]=\{2\}$ for all $n$, the Bayesian simple regret of qEI is given by
\begin{align*}
    \E\left[f(x^*) - f(2)\right] &= \sum_{i=1}p_i\left(\max_{x\in\X}f_i(x) - f_i(2)\right)\\
    &= p
\end{align*}
for all $n$.
\end{proof}

\section{ADDITIONAL EMPIRICAL RESULTS}
\label{sec:addit}
Besides the experiments we present in the main paper, we additionally present experimental results with $q=4$ vs $q=6$ and investigate different acquisition functions' performance under various noise levels for $q=2$ case.

\subsection{Results for $q=4$ vs. $q=6$}
Figure~\ref{fig:q3_results} shows the experiment results for the best performing three acquisition (qTS, qEI, and qEUBO) functions with $q=4$ and $q=6$ over 150 queries.
As discussed in the paper, $q=6$ offers little improvement over $q=4$ for all acquisition functions.
\begin{figure*}[ht]
\centering
\begin{tabular}[b]{c}%
\includegraphics[width=0.317\textwidth]{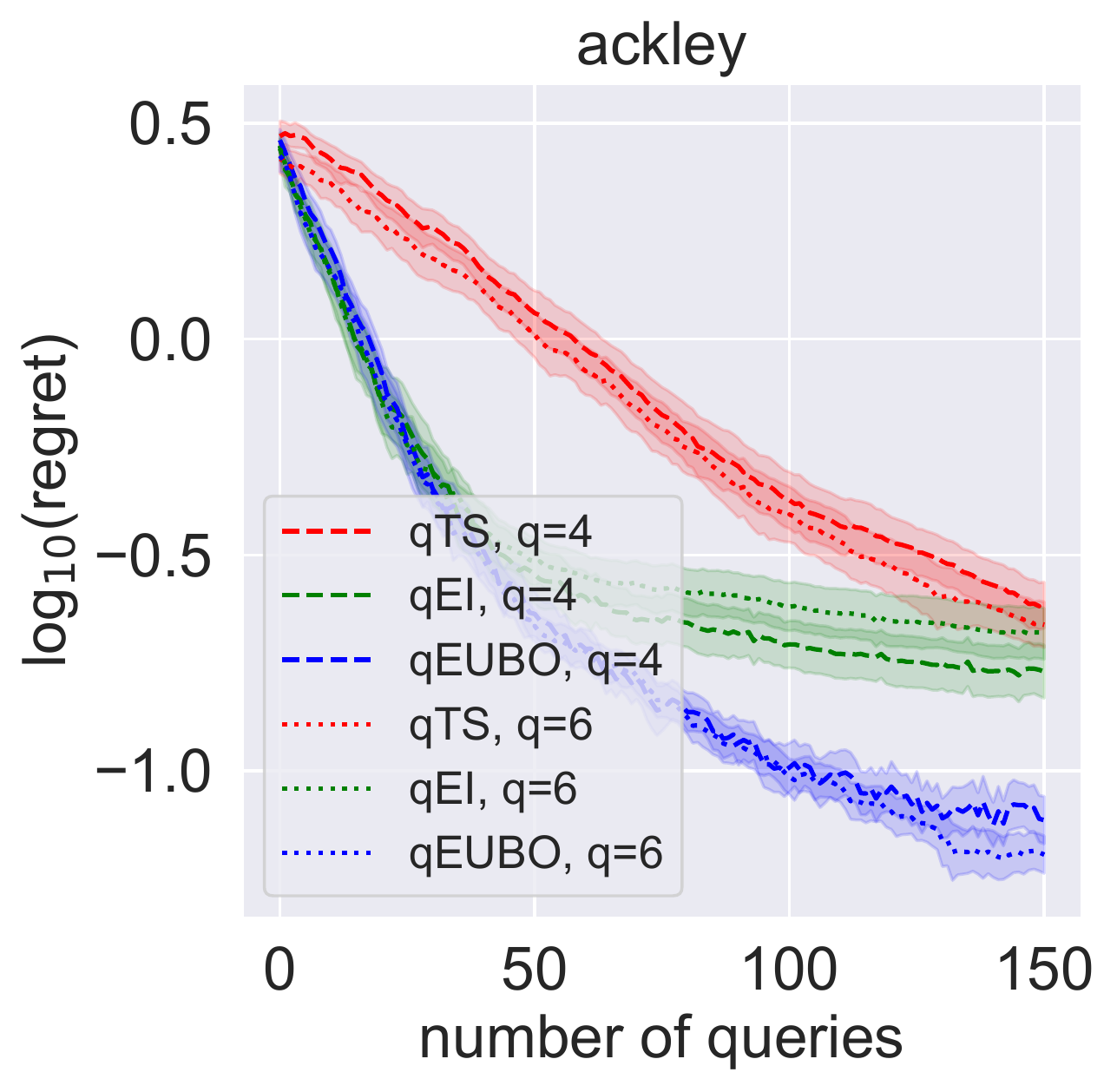}
  \includegraphics[width=0.32\textwidth]{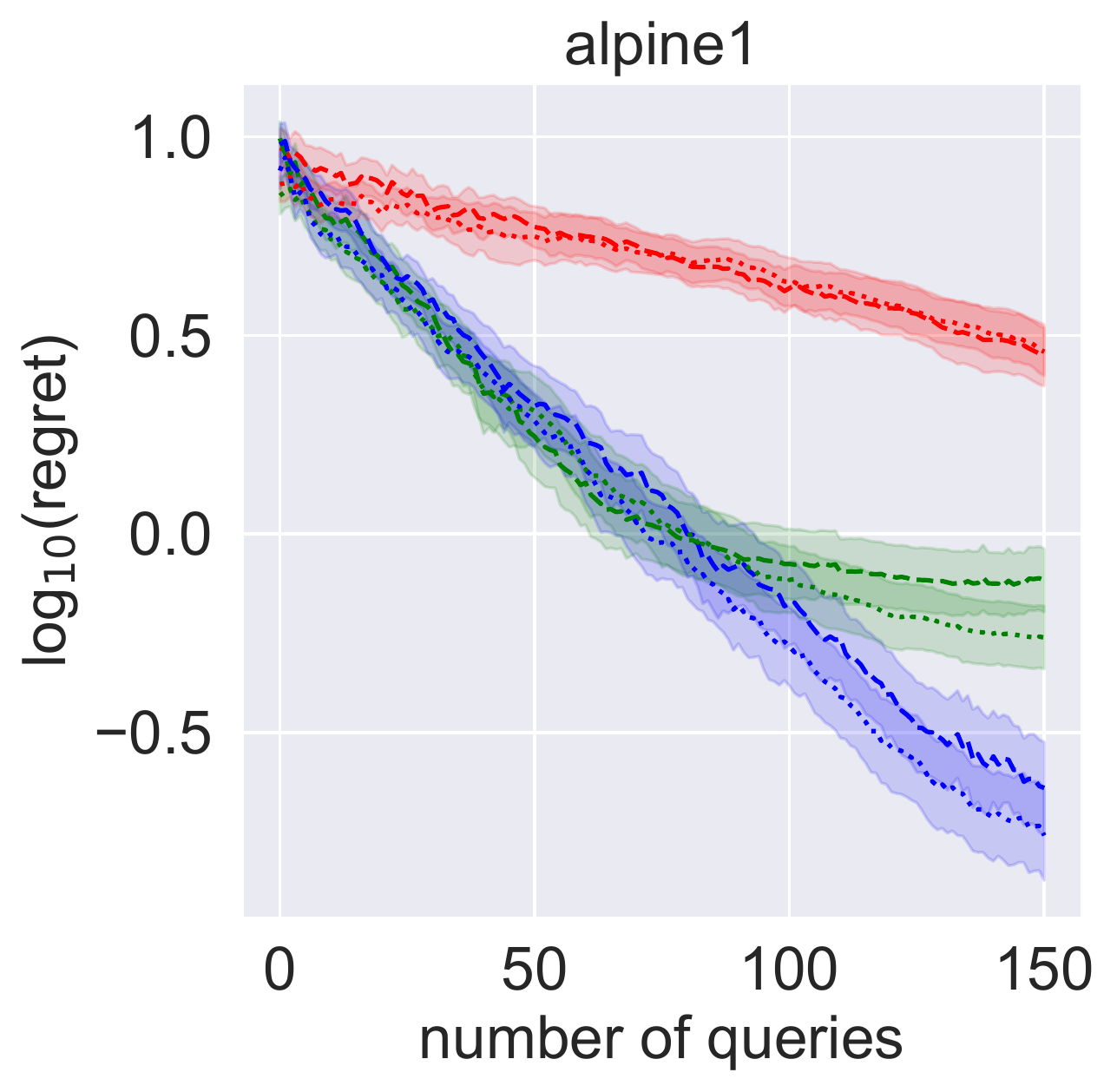}
  \includegraphics[width=0.32\textwidth]{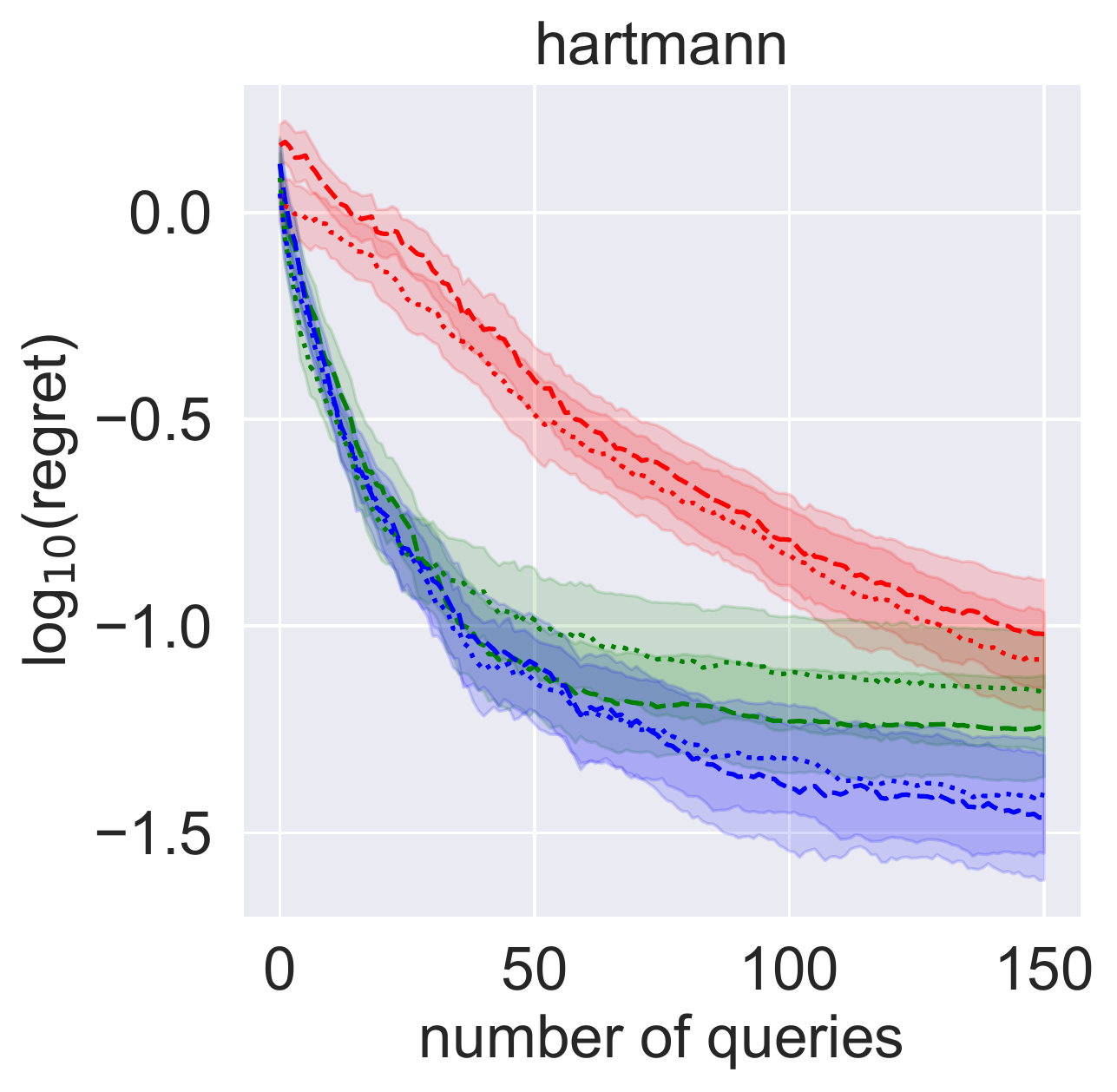}\\
  \includegraphics[width=0.317\textwidth]{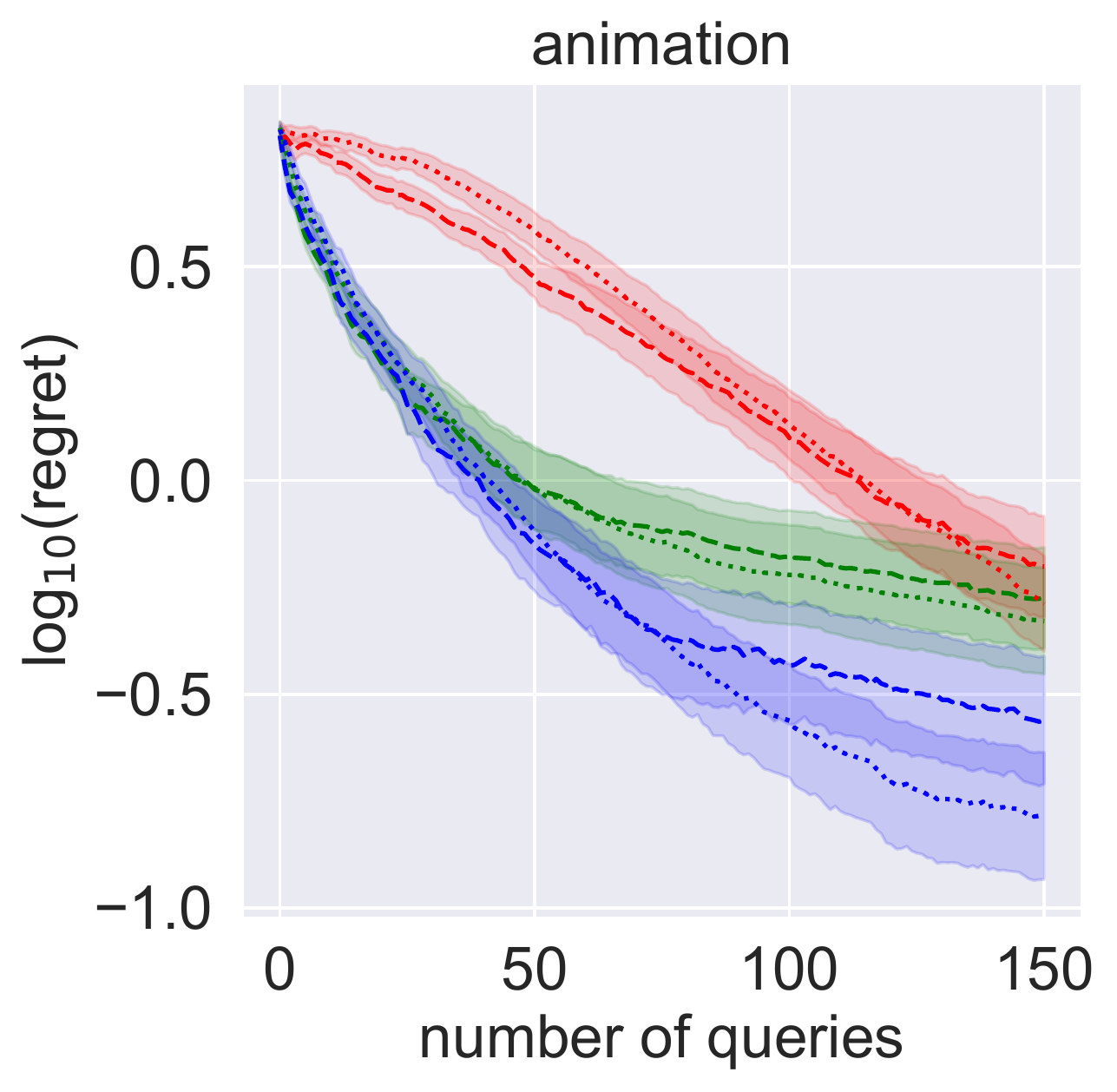}
\includegraphics[width=0.32\textwidth]{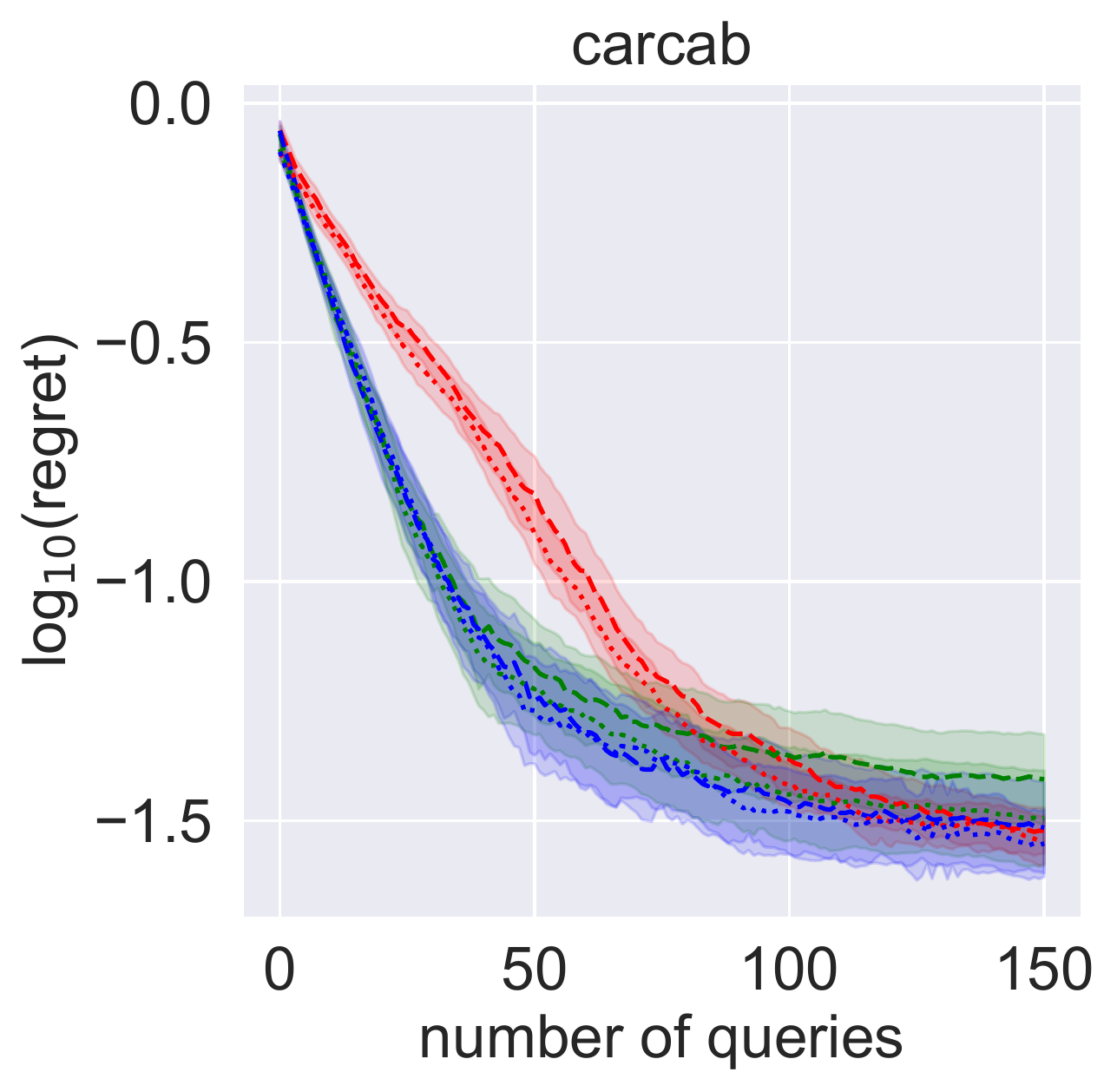}
  \includegraphics[width=0.32\textwidth]{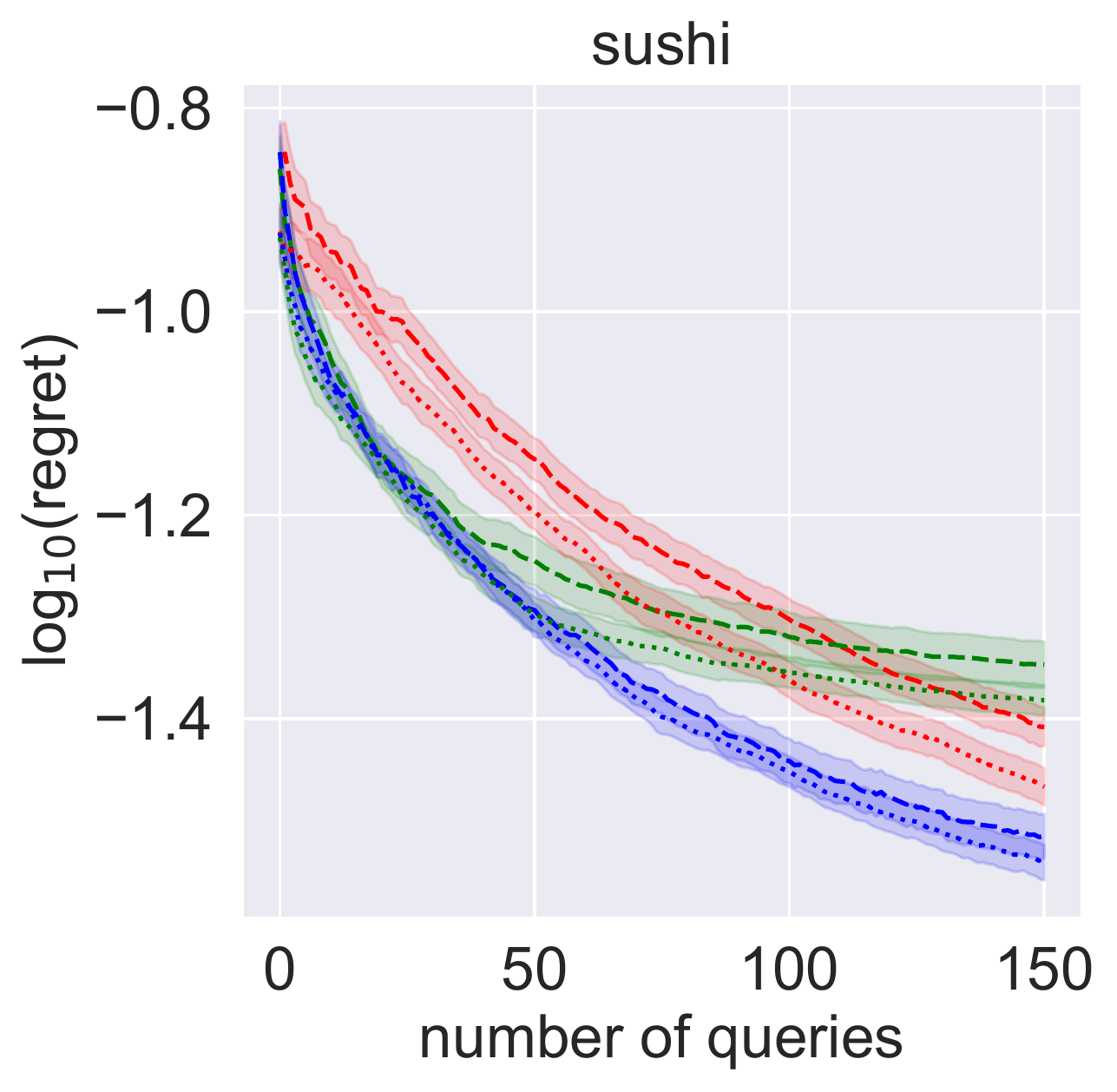}
 \end{tabular}
 \caption{log10(optimum value - objective value at the maximizer of the posterior mean) using moderate logistic noise and $q=4$ and $q=6$ alternatives per DM query.
 \label{fig:q3_results}}
\end{figure*}

\newpage
\subsection{Experiments with Varying Levels of Noise}
Figure~\ref{fig:noise_results} shows the results of three test problems (Ackley, Animation, and Sushi) by injecting varying levels of comparison noise.
Left, center, and right columns show the results for low, middle, and high noise levels respectively.
The noise levels are chosen so that the DM makes a comparison mistake 10, 20, and 30\% of the time on average when asked to compare random pairs of points among those with the top 1\% function values within the optimization domain $\X$ under a logistic likelihood.
Those noise levels are estimated over a large grid of random points in $\X$.

For experiments with lower noise, the performance gap between $\eubo$ and other baseline methods is most pronounced. This is consistent with Theorem~\ref{thm:2}, which shows that qEUBO is a better approximation of the one-step Bayes optimal policy for lower noise levels.
With increasing noise, the performance of all methods decreases. However, we consistently observe superior performance of $\eubo$ compared to other baseline methods. These results are averaged over 50 replications for Ackley, and 100 replications for Animation and Sushi.
\begin{figure*}[ht]
\centering
\begin{tabular}[b]{c}%
\includegraphics[width=0.32\textwidth]{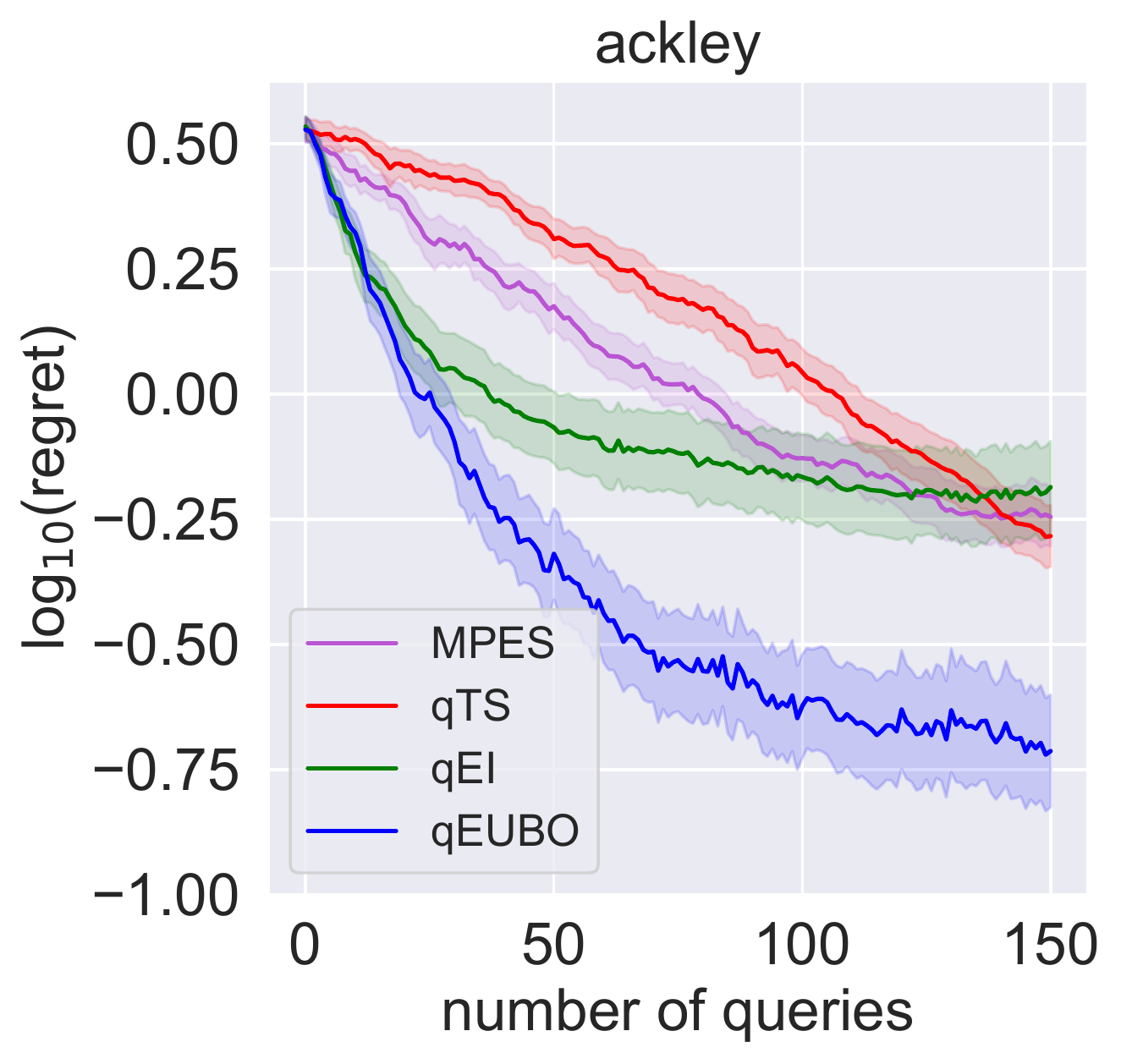}
  \includegraphics[width=0.32\textwidth]{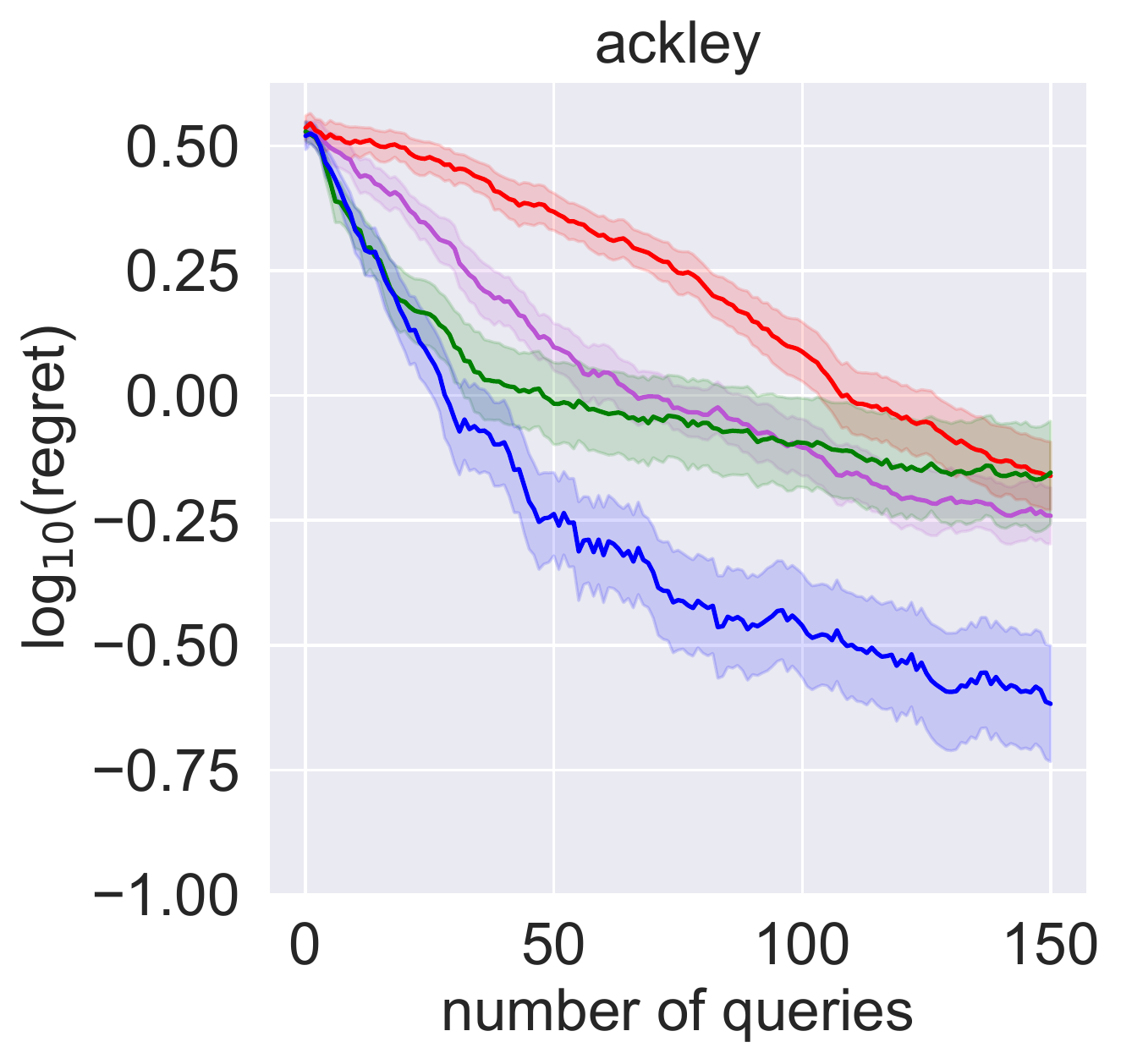}
   \includegraphics[width=0.32\textwidth]{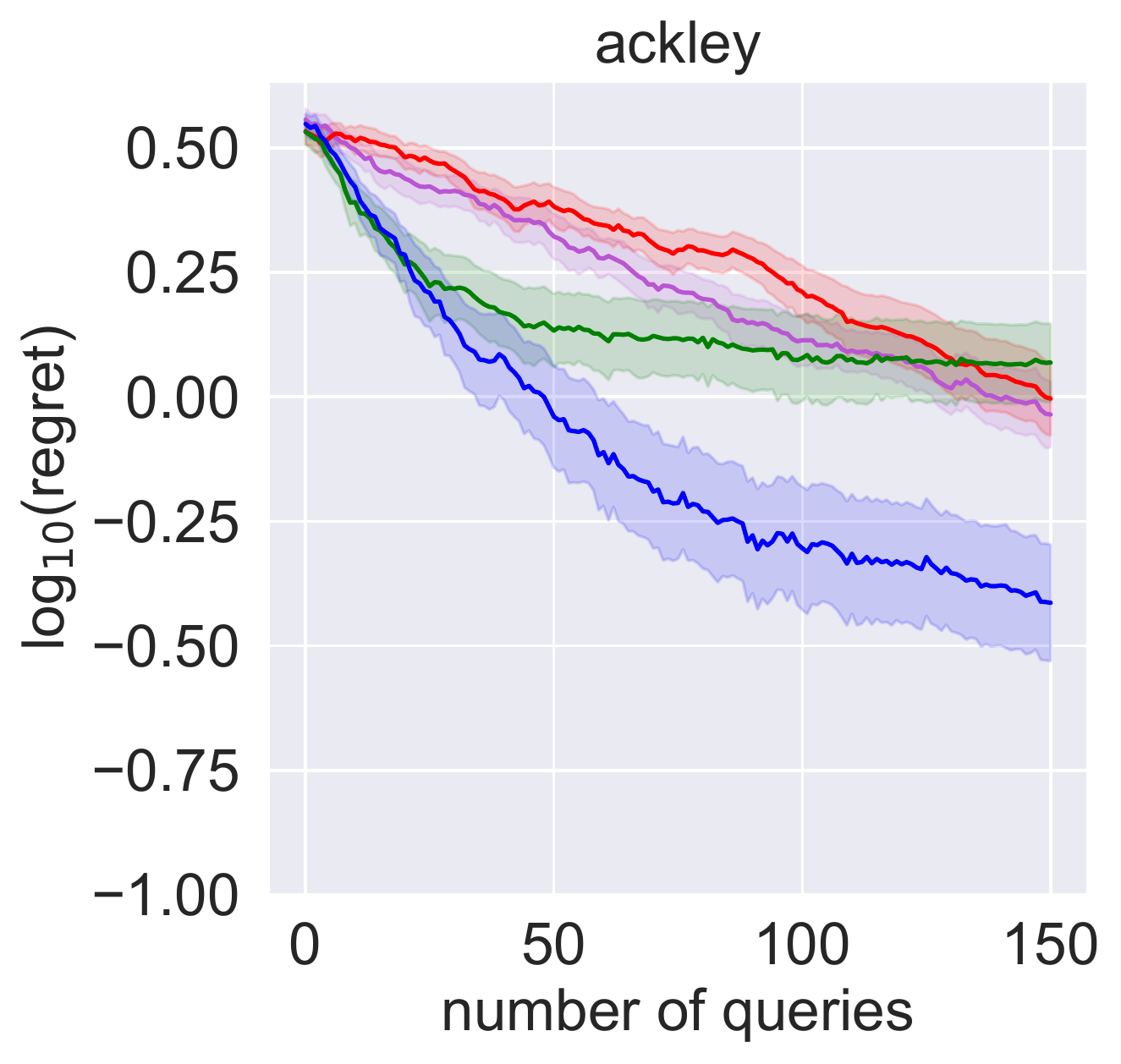}\\
\includegraphics[width=0.32\textwidth]{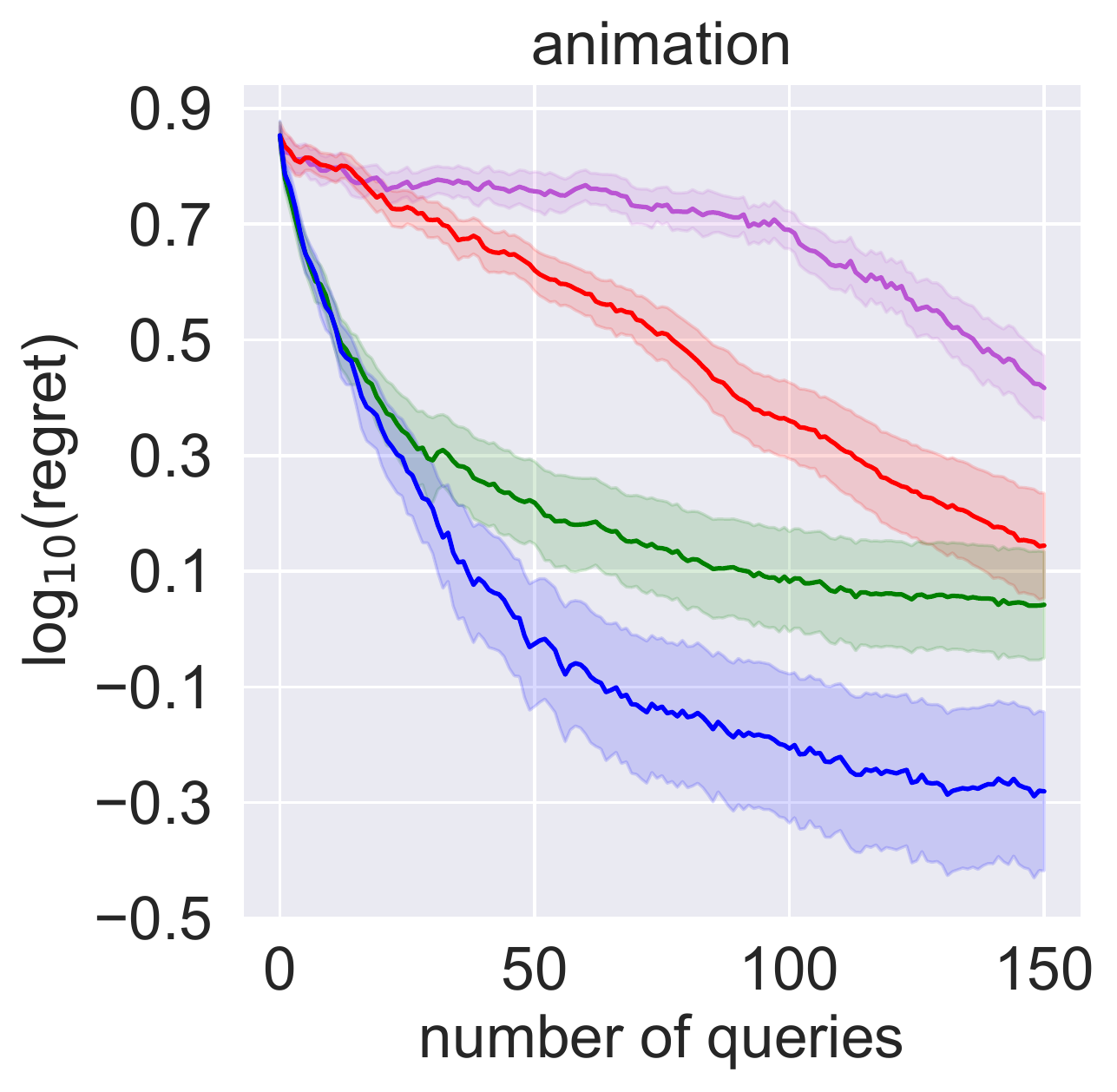}
  \includegraphics[width=0.32\textwidth]{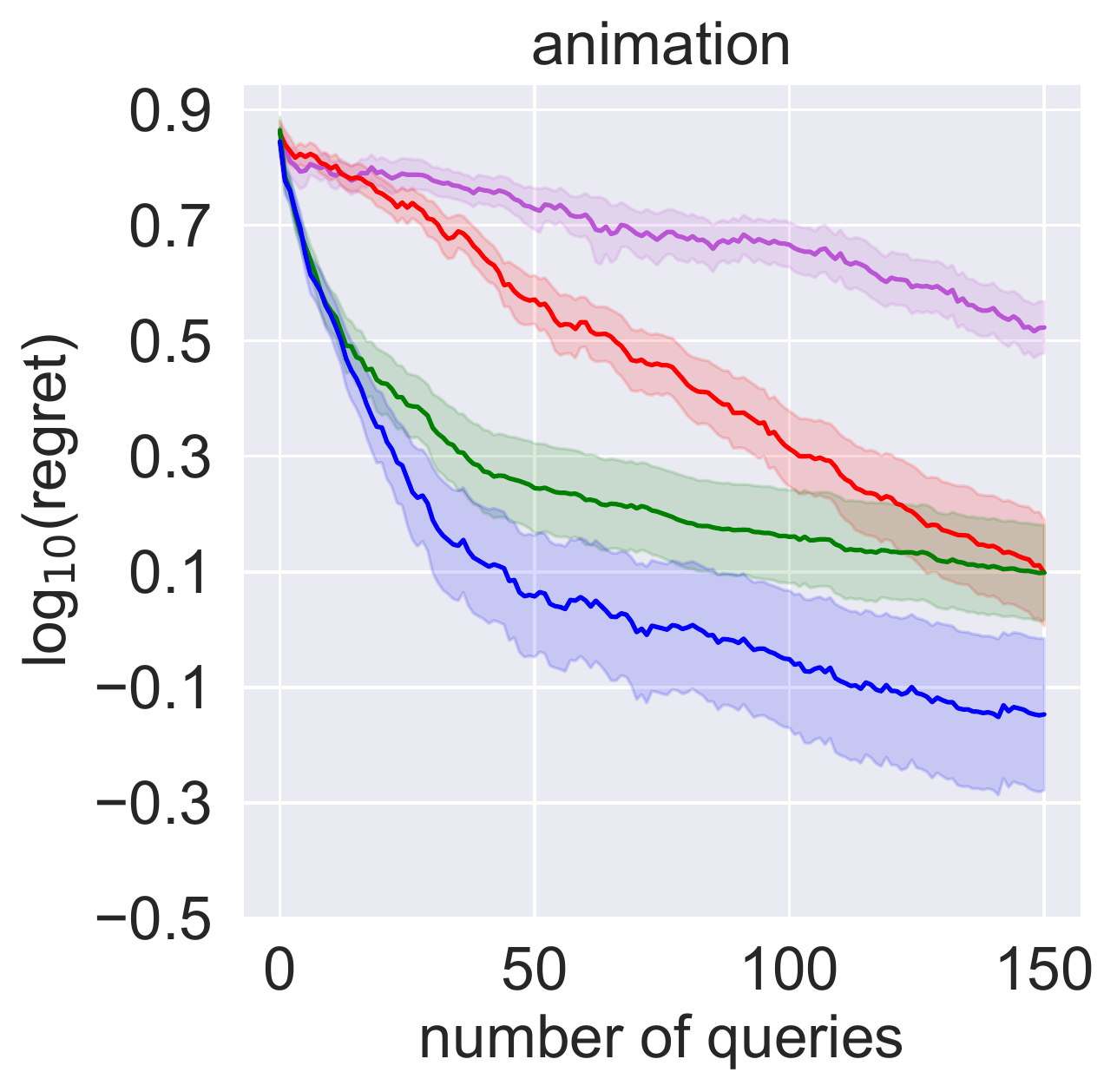}
   \includegraphics[width=0.32\textwidth]{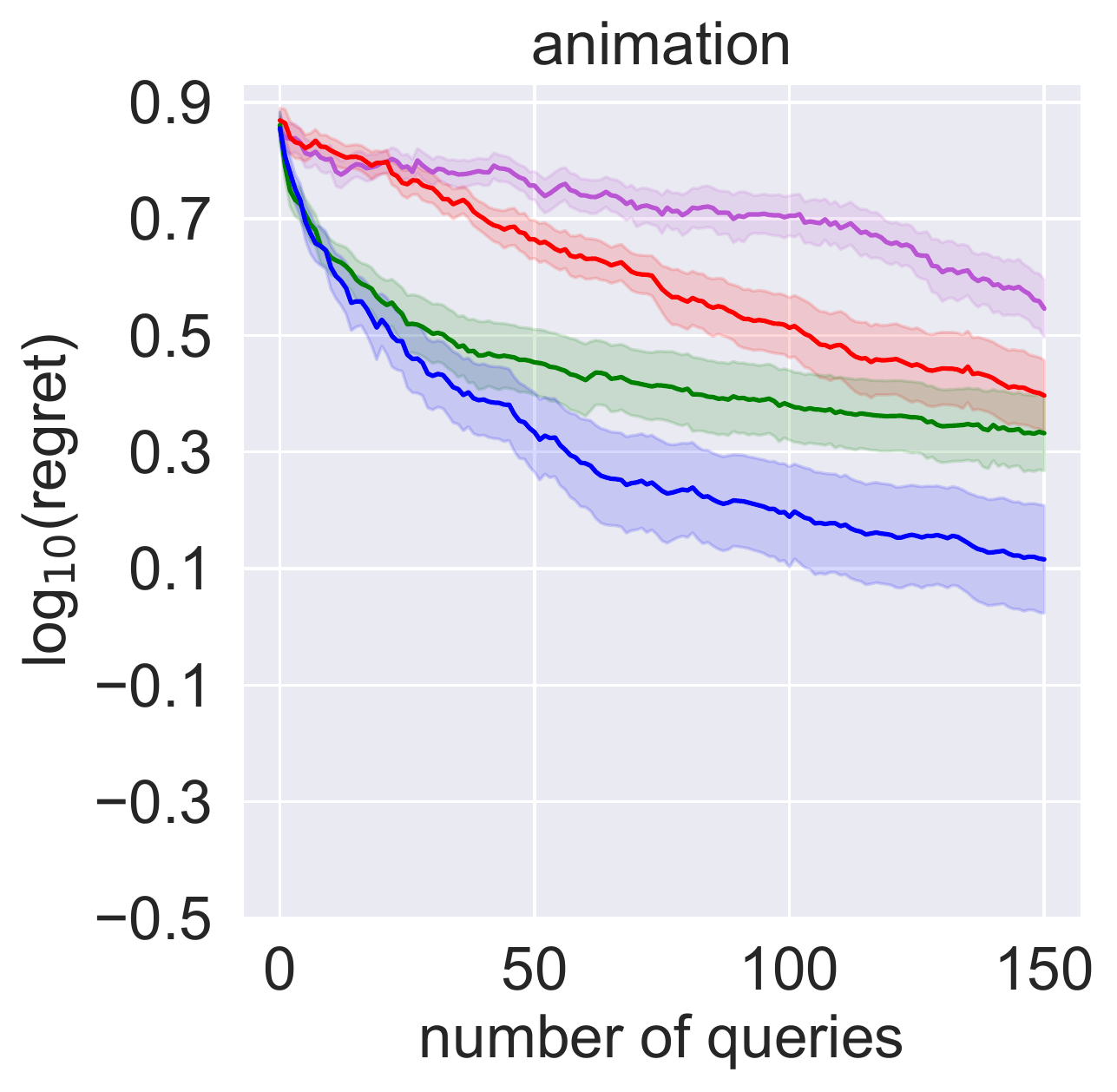}
\\
\includegraphics[width=0.32\textwidth]{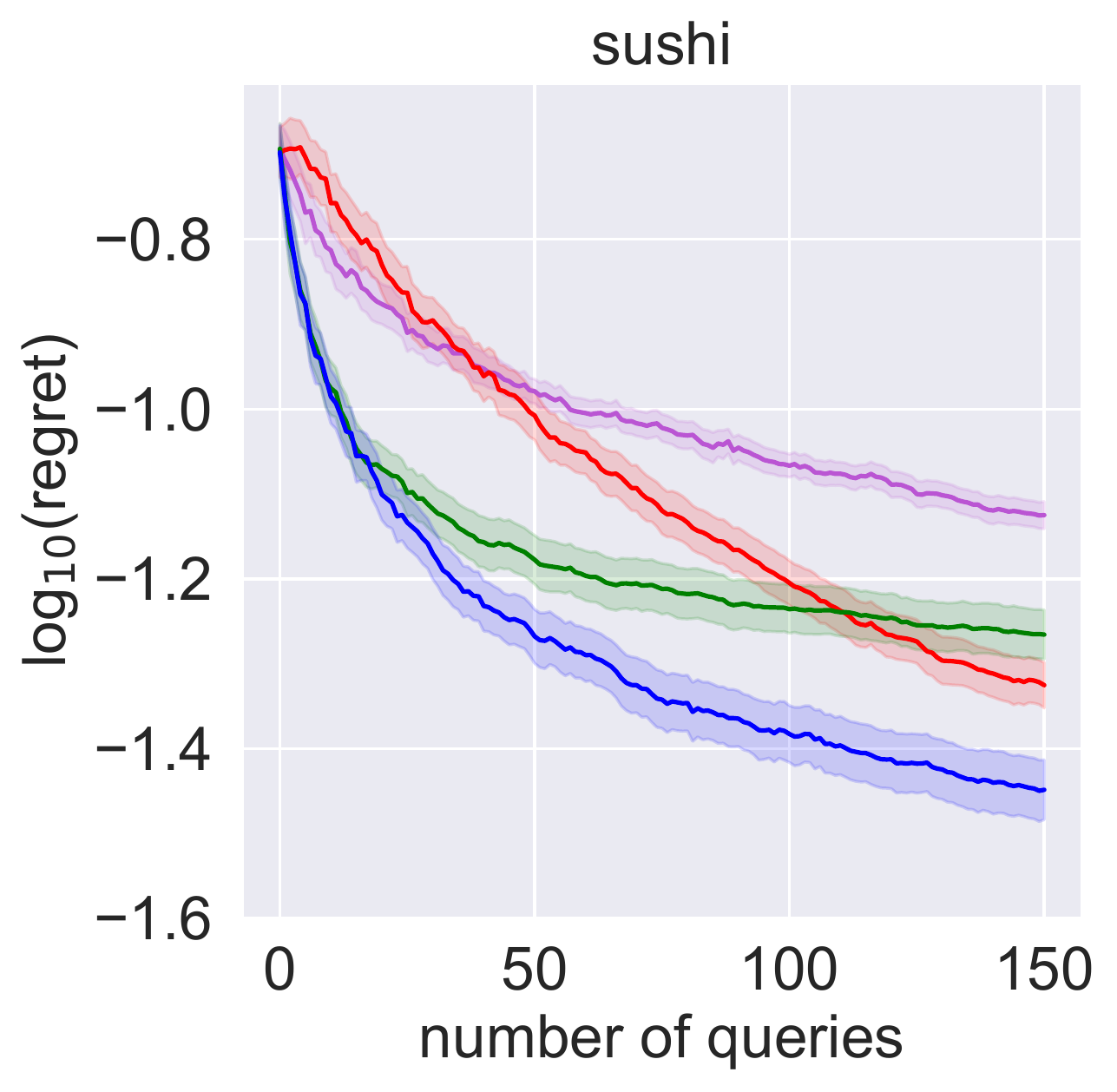}
  \includegraphics[width=0.32\textwidth]{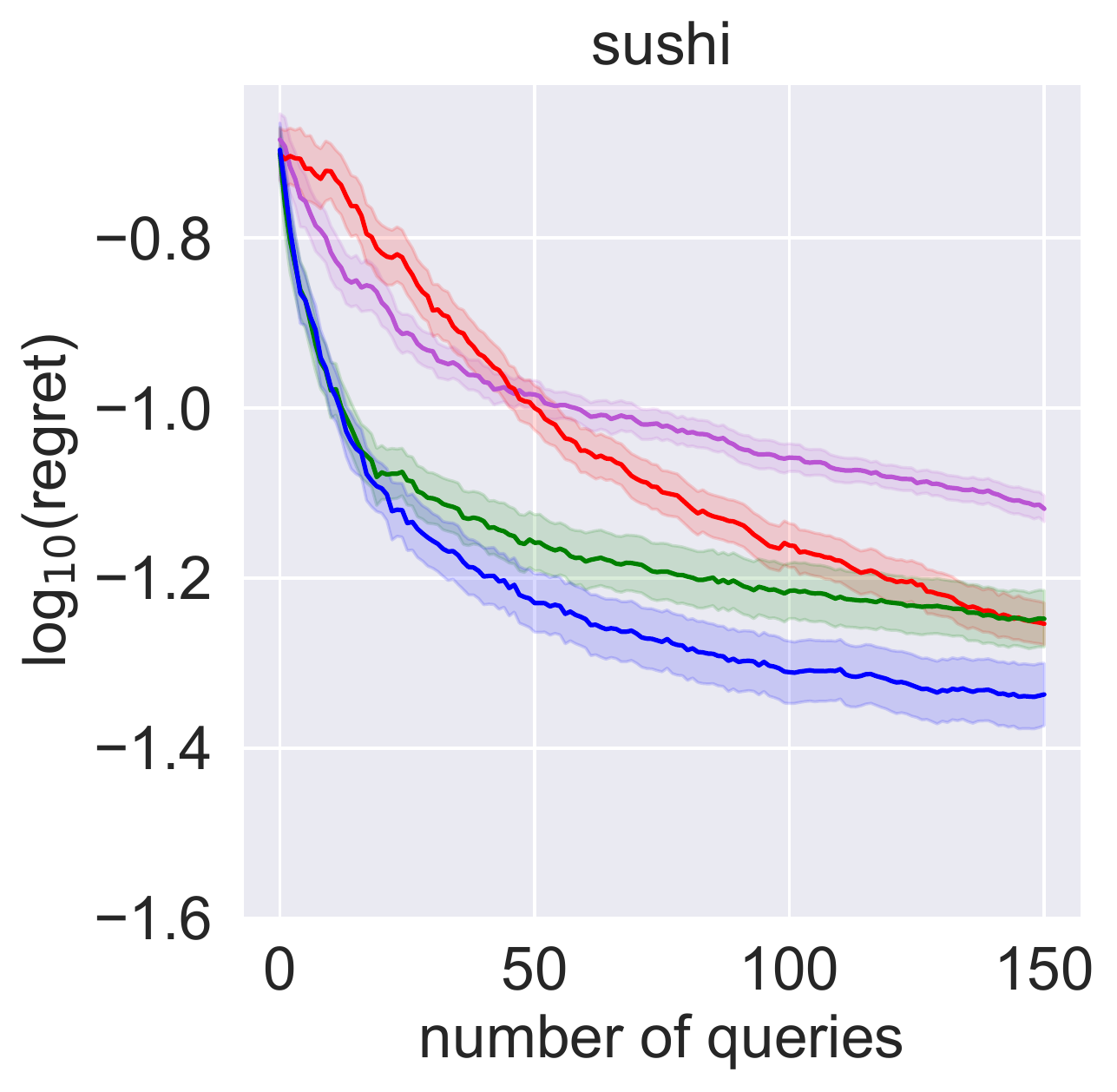}
   \includegraphics[width=0.32\textwidth]{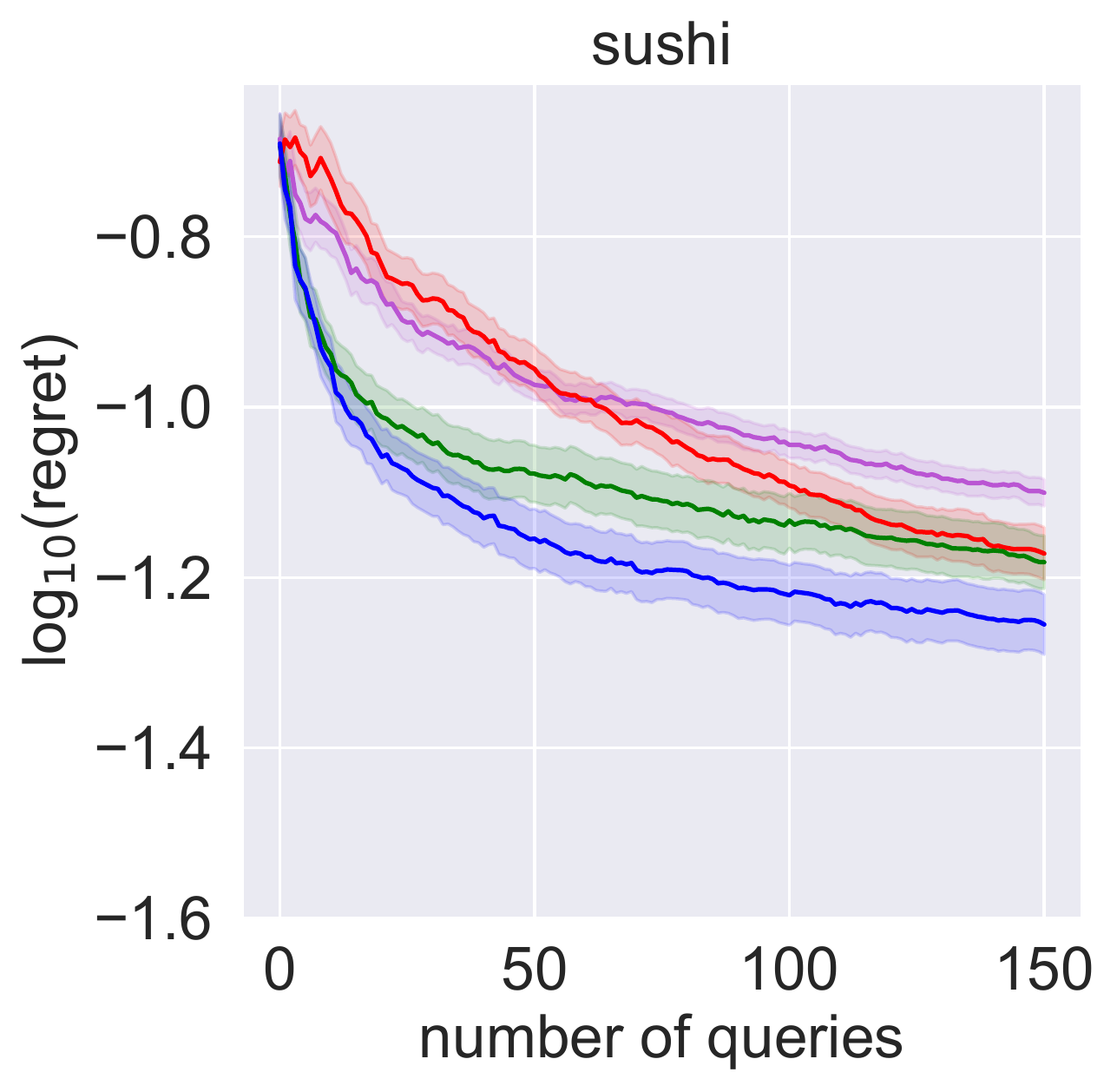}
 \end{tabular}
 \caption{log10(optimum value - objective value at the maximizer of the posterior mean) with $q=2$ alternatives per DM query and varying levels of noise. Left, center, and right columns show the results for three of our experiments with low, middle, and high noise, respectively.
 \label{fig:noise_results}}
\end{figure*}

\vfill

\end{document}